\documentclass[envcountsame]{llncs}

\usepackage{srcltx}

\usepackage{times}

\usepackage{amsmath, amssymb}

\usepackage{braket}

\usepackage{stmaryrd}

\usepackage{comment}

\newcommand{\an}[1]{\left\langle #1 \right\rangle}

\newcommand{\C}[1]{\textsf{C}\protect\nobreakdash#1\hspace{0pt}}
\newcommand{\mC}{\mathsf{C}}
\newcommand{\SE}[1]{\textsf{SE}\protect\nobreakdash#1\hspace{0pt}}
\newcommand{\mSE}{\mathsf{SE}}

\newcommand{\mSt}{\mathsf{S}}
\newcommand{\SMR}[1]{\textsf{SMR}\protect\nobreakdash#1\hspace{0pt}}
\newcommand{\mSMR}{\mathsf{SMR}}
\newcommand{\SR}[1]{\textsf{SR}\protect\nobreakdash#1\hspace{0pt}}
\newcommand{\mSR}{\mathsf{SR}}
\newcommand{\SU}[1]{\textsf{SU}\protect\nobreakdash#1\hspace{0pt}}
\newcommand{\mSU}{\mathsf{SU}}
\newcommand{\SUC}[1]{\textsf{SUC}\protect\nobreakdash#1\hspace{0pt}}

\newcommand{\lang}{\mathcal{L}}
\newcommand{\pint}{\mathcal{I}}
\newcommand{\seint}{\mathcal{I}^{\mSE}}

\newcommand{\lif}{\subset}
\newcommand{\lthen}{\supset}
\newcommand{\lequiv}{\equiv}
\newcommand{\bigland}{\bigwedge}
\newcommand{\biglor}{\bigvee}

\newcommand{\lpnot}[1][\!\!]{\sim#1}
\newcommand{\lpif}{\leftarrow}

\newcommand{\at}{p}

\newcommand{\modc}[1]{\mathsf{mod}_{\mC}\left(#1\right)}
\newcommand{\modse}[1]{\mathsf{mod}_{\mSE}\left(#1\right)}
\newcommand{\mcS}{\ensuremath{\mathcal{S}}}
\newcommand{\mcT}{\ensuremath{\mathcal{T}}}

\newcommand{\mcSh}[1]{\ensuremath{\mathcal{S}}\protect\nobreakdash-\hspace{0pt}}
\newcommand{\SPart}[2]{%
	\ensuremath{\mathcal{S}}%
	\protect\nobreakdash-\hspace{0pt}%
	#1%
	\protect\nobreakdash-\hspace{0pt}%
	#2%
}
\newcommand{\SPosHead}{\SPart{positive}{head}}
\newcommand{\SNegHead}{\SPart{negative}{head}}
\newcommand{\SPosBody}{\SPart{positive}{body}}
\newcommand{\SNegBody}{\SPart{negative}{body}}

\newcommand{\synt}[1]{\mathsf{rule}(#1)}

\newcommand{\secan}[1]{\mathsf{can}(#1)}

\newcommand{\pr}[1][\defpr]{\mathcal{#1}}
\newcommand{\prP}{\pr[P]}
\newcommand{\prQ}{\pr[Q]}

\def\qed{\ifmmode\tag*{\squareforqed}\else{\unskip\nobreak\hfil
\penalty50\hskip1em\null\nobreak\hfil\squareforqed
\parfillskip=0pt\finalhyphendemerits=0\endgraf}\fi}

\newenvironment{lemma*}[1]{\par \noindent \textbf{Lemma \ref{#1}.}\it}{}
\newenvironment{theorem*}[1]{\par \noindent \textbf{Theorem \ref{#1}.}\it}{}
\newenvironment{proposition*}[1]{\par \noindent \textbf{Proposition \ref{#1}.}\it}{}

\newenvironment{extended}{}{}

\begin{document}

\title{Back and Forth Between Rules and \SE-Models %
 (Extended Version)\thanks{This is an extended version of the paper accepted for publication at LPNMR 2011. ~ ~ ~ ~ ~ ~ ~ ~ ~ \textbf{Changes on March 1, 2011}: minor substitutions to be in line with the LPNMR version.}%
}
\author{Martin Slota \and Jo{\~a}o Leite}
\institute{
	CENTRIA \& Departamento de Inform{\'a}tica \\
	Universidade Nova de Lisboa \\
	Quinta da Torre \\
	2829-516 Caparica, Portugal
}

\maketitle

\begin{abstract}
Rules in logic programming encode information about mutual interdependencies
between literals that is not captured by any of the commonly used semantics.
This information becomes essential as soon as a program needs to be modified
or further manipulated.

We argue that, in these cases, a program should not be viewed solely as the
set of its models. Instead, it should be viewed and manipulated as the
\emph{set of sets of models} of each rule inside it. With this in mind, we
investigate and highlight relations between the \SE-model semantics and
individual rules. We identify a set of representatives of rule equivalence
classes induced by \SE-models, and so pinpoint the exact expressivity of this
semantics with respect to a single rule. We also characterise the class of
sets of \SE-interpretations representable by a single rule. Finally, we
discuss the introduction of two notions of equivalence, both stronger than
\emph{strong equivalence} \cite{Lifschitz2001} and weaker than \emph{strong
update equivalence} \cite{Inoue2004}, which seem more suitable whenever the
dependency information found in rules is of interest.
\end{abstract}

\section{Motivation}

In this paper we take a closer look at the relationship between the \SE-model
semantics and individual rules of a logic program. We identify a set of
representatives of rule equivalence classes, which we dub \emph{canonical
rules}, characterise the class of sets of \SE-interpretations that are
representable by a single rule, and show how the corresponding canonical rules
can be reconstructed from them. We believe that these results pave the way to
view and manipulate a logic program as the \emph{set of sets of \SE-models} of
each rule inside it. This is important in situations when the set of
\SE-models of the whole program fails to capture essential information encoded
in individual rules inside it, such as when the program needs to be modified
or further manipulated. With this in mind, we briefly discuss two new notions
of equivalence, stronger than \emph{strong equivalence} \cite{Lifschitz2001}
and weaker than \emph{strong update equivalence} \cite{Inoue2004}.

In many extensions of Answer-Set Programming, individual rules of a program
are treated as first-class citizens -- apart from their prime role of encoding
the answer sets assigned to the program, they carry essential information
about mutual interdependencies between literals that cannot be captured by
answer sets. Examples that enjoy these characteristics include the numerous
approaches that deal with dynamics of logic programs, where inconsistencies
between older and newer knowledge need to be resolved by ``sacrificing'' parts
of an older program (such as in
\cite{Damasio1997,Alferes2000,Eiter2002,Sakama2003,Zhang2006,Alferes2005,Delgrande2007,Delgrande2008,Delgrande2010}).
These approaches look at subsets of logic programs in search of plausible
conflict resolutions. Some of them go even further and consider particular
literals in heads and bodies of rules in order to identify conflicts and find
ways to solve them. This often leads to definitions of new notions which are
\emph{too} syntax-dependent. At the same time, however, semantic properties of
the very same notions need to be analysed, and their syntactic basis then
frequently turns into a serious impediment.

Arguably, a \emph{more} syntax-independent method for this kind of operations
would be desirable. Not only would it be theoretically more appealing, but it
would also allow for a better understanding of its properties with respect to
the underlying semantics. Moreover, such a more semantic approach could
facilitate the establishment of bridges with the area of Belief Change (see
\cite{Gardenfors1992} for an introduction), and benefit from the many years of
research where semantic change operations on monotonic logics have been
studied, desirable properties for such operations have been identified, and
constructive definitions of operators satisfying these properties have been
introduced. 

However, as has repeatedly been argued in the literature
\cite{Alferes2000,Slota2010b}, fully semantic methods do not seem to be
appropriate for the task at hand. Though their definition and analysis is
technically possible and leads to very elegant and seemingly desirable
properties, there are a number of simple examples for which these methods fail
to provide results that would be in line with basic intuitions
\cite{Alferes2000}. Also, as shown in \cite{Slota2010b}, these individual
problems follow a certain pattern: intuitively, any purely semantic approach
to logic program updates satisfying a few very straightforward and desirable
properties cannot comply with the property of \emph{support}
\cite{Apt1988,Dix1995a}, which lies at the very heart of semantics for Logic
Programs. This can be demonstrated on simple programs $\prP = \set{p., q.}$
and $\prQ = \set{p., q \lpif p.}$ which are \emph{strongly equivalent}, thus
indistinguishable from the semantic perspective, but while $\prP$ does not
contain any dependencies, $\prQ$ introduces a dependence of atom $q$ upon atom
$p$. This has far-reaching consequences, at least with respect to important
notions from the logic programming point of view, such as that of
\emph{support}, which are themselves defined in syntactic rather than semantic
terms. For example, if we change our beliefs about $p$, and come to believe
that it is false, we may expect different beliefs regarding $q$, depending on
whether we start form $\prP$, in which case $q$ would still be true, or
$\prQ$, in which case $q$ would no longer be true because it is no longer
supported.

We believe that rules indeed contain information that, to the best of our
knowledge, cannot be captured by any of the existing semantics for Logic
Programs. In many situations, this information is essential for making further
decisions down the line. Therefore, any operation on logic programs that is
expected to respect syntax-based properties like \emph{support} cannot operate
solely on the semantic level, but rather has to look inside the program and
acknowledge rules as the atomic pieces of knowledge. At the same time,
however, rules need not be manipulated in their original form. The abstraction
provided by Logic Programming semantics such as \SE-models can be used to
discard unimportant differences between the syntactic forms of rules and focus
on their semantic content. Thus, while a program cannot be viewed as the set
of its models for reasons described above, it can still be viewed as a
\emph{set of sets of models} of rules in it. Such a shift of focus should make
the approach easier to manage theoretically, while not neglecting the
importance of literal dependencies expressed in individual rules. It could
also become a bridge between existing approaches to rule evolution and
properties as well as operator constructions known from Belief Change, not
only highlighting the differences between them, but also clarifying why such
differences arise in the first place.

However, before a deeper investigation of such an approach can begin, we do
need to know more about the relation of \SE-models and individual rules. This
is the aim of this paper, where we:

\begin{itemize}
	\item identify a set of representatives of rule equivalence classes induced
		by the \SE-model semantics, which we dub \emph{canonical rules};

	\item show how to reconstruct \emph{canonical rules} from their sets of
		\SE-models;

	\item based on the above, characterise the sets of \SE-interpretations that
		are representable by a single rule;

	\item reveal connections between the set of \SE-models of a rule and convex
		sublattices of the set of classical interpretations;

	\item introduce two new notions of equivalence -- stronger than \emph{strong
		equivalence} \cite{Lifschitz2001} and weaker than \emph{strong update
		equivalence} \cite{Inoue2004} -- and argue that they are more suitable
		when rules are to be treated as first-class citizens.
\end{itemize}

We believe that these results provide important insights into the workings of
\SE-models with respect to individual rules and will serve as a toolset for
manipulating logic programs at the semantic level.

The rest of this document is structured as follows: We introduce syntax and
semantics of logic programs in Sect. \ref{sec:preliminaries} while in Sect.
\ref{sec:se models of rules} we define the set of representatives for rule
equivalence classes and introduce transformations pinpointing the expressivity
of \SE-model semantics with respect to individual rules. We also give two
characterisations of the sets of \SE-interpretations that are representable by
a single rule. In Sect. \ref{sec:discussion} we discuss the relevance of our
results and propose the two new notions of equivalence.

\section{Preliminaries}

\label{sec:preliminaries}

We assume to be given a nonempty, finite set of propositional atoms $\lang$
from which we construct both propositional formulae and rules.

\emph{Propositional formulae} are formed in the usual way from propositional
atoms in $\lang$, the logical constants $\top$ an $\bot$, and the connectives
$\lnot, \land, \lor, \lif, \lthen, \lequiv$. An \emph{interpretation} is any
subset of $\lang$, naturally inducing a truth assignment to all propositional
formulae. If a formula $\phi$ is true under interpretation $I$, we also say
that $I$ is a \emph{model of $\phi$}. The set of all interpretations is
denoted by $\pint$.

Similarly as for propositional formulae, the basic syntactic building blocks
of rules are propositional atoms from $\lang$. A \emph{negative literal} is an
atom preceded by $\lpnot{}$, denoting default negation. A \emph{literal} is
either an atom or a negative literal. As a convention, double default negation
is absorbed, so that $\lpnot[]\lpnot p$ denotes the atom $p$. Given a set of
literals $X$, we introduce the following notation:
\begin{align*}
	X^+ &= \Set{ p \in \lang | p \in X }
	& X^- &= \Set{ p \in \lang | \lpnot p \in X }
	& \lpnot X &= \Set{ \lpnot p | p \in X \cap \lang }
\end{align*}

Given natural numbers $k, l, m, n$ and atoms $p_1, \dotsc, p_k$, $q_1, \dotsc,
q_l$, $r_1, \dotsc, r_m$, $s_1, \dotsc, s_n$, we say the pair of sets of
literals
\begin{equation} \label{eq:rule_formal}
	\an{
		\Set{p_1, \dotsc, p_k, \lpnot q_1, \dotsc, \lpnot q_l},
		\Set{r_1, \dotsc, r_m, \lpnot s_1, \dotsc, \lpnot s_n}
	}
\end{equation}
is a \emph{rule}. The first component of a rule \eqref{eq:rule_formal} is
denoted by $H(r)$ and the second by $B(r)$.  We say $H(r)$ is the \emph{head
of $r$}, $H(r)^+$ is the \emph{positive head of $r$}, $H(r)^-$ is the
\emph{negative head of $r$}, $B(r)$ is the \emph{body of $r$}, $B(r)^+$ is the
\emph{positive body of $r$} and $B(r)^-$ is the \emph{negative body of $r$}.
Usually, for convenience, instead of a rule $r$ of the form
\eqref{eq:rule_formal} we write the expression
\begin{equation} \label{eq:rule}
	p_1; \dotsc; p_k; \lpnot q_1; \dotsc; \lpnot q_l
		\lpif r_1, \dotsc, r_m, \lpnot s_1, \dotsc, \lpnot s_n.
\end{equation}
or, alternatively, $H(r)^+; \lpnot H(r)^- \lpif B(r)^+, \lpnot B(r)^-$. A rule
is called \emph{positive} if its head and body contain only atoms. A
\emph{program} is any set of rules.

We also introduce the following non-standard notion which we will need
throughout the rest of the paper:

\begin{definition}[Canonical Tautology] \label{def:canonical tautology}
	Let $\at_\varepsilon$ be an arbitrary but fixed atom. The \emph{canonical
	tautology}, denoted by $\varepsilon$, is the rule $\at_\varepsilon \lpif
	\at_\varepsilon$.
\end{definition}

In the following, we define two semantics for rules. One is that of classical
models, where a rule is simply treated as a classical implication. The other
is based on the logic of Here-and-There \cite{Lukasiewicz1941,Pearce1997},
more accurately on a reformulation of the here-and-there semantics, called
\emph{\SE-model semantics}, defined for rules \cite{Turner2003}. This second
semantics is strictly more expressive than both classical models and the
stable model semantics \cite{Gelfond1988}.

\begin{extended}
\paragraph{\C-Models.}~\!\!
\end{extended}
We introduce the classical model of a rule by translating the rule into a
propositional formula: Given a rule $r$ of the form \eqref{eq:rule}, we define
the propositional formula $\overline{r}$ as $\biglor \Set{p_1, \dotsc, p_k,
\lnot q_1, \dotsc, \lnot q_l} \lif \bigland \Set{r_1, \dotsc, r_m, \lnot s_1,
\dotsc, \lnot s_n}$.  Note that $\biglor \emptyset \equiv \bot$ and $\bigland
\emptyset \equiv \top$. A classical model, or \emph{\C-model}, of a rule $r$
is any model of the formula $\overline{r}$.
\begin{extended}
The set of all \C-models of a rule $r$ is denoted by $\modc{r}$.
\end{extended}

\begin{extended}
We say a rule $r$ is a \emph{\C-tautology}, or \emph{\C-tautological}, if
$\modc{r} = \pint$. Otherwise, $r$ is \emph{non-\C-tautological}. We say two
rules $r, r'$ are \emph{\C-equivalent} if they have the same set of \C-models.
Note that the canonical tautology $\varepsilon$ (c.f. Definition
\ref{def:canonical tautology}) is \C-tautological.
\end{extended}

\begin{extended}
\paragraph{\SE-Models.}~\!\!
\end{extended}
Given a rule $r$ and an interpretation $J$, we define the \emph{reduct of $r$
relative to $J$}, denoted by $r^J$, as follows: If some atom from $H(r)^-$ is
false under $J$ or some atom from $B(r)^-$ is true under $J$, then $r^J$ is
$\varepsilon$; otherwise $r^J$ is $H(r)^+ \lpif B(r)^+$. Intuitively, the
reduct $r^J$ is the positive part of a rule $r$ that ``remains'' after all its
negative literals are interpreted under interpretation $J$. The two conditions
in the definition check whether the rule is satisfied based on the negative
atoms in its head and body, interpreted under $J$. If this is the case, the
reduct is by definition the canonical tautology. If none of these conditions
is satisfied, the positive parts of $r$ are kept in the reduct, discarding the
negative ones.

An \emph{\SE-interpretation} is a pair of interpretations $\an{I, J}$ such
that $I$ is a subset of $J$. The set of all \SE-interpretations is denoted by
$\seint$. We say that an \SE-interpretation $\an{I, J}$ is an \emph{\SE-model}
of a rule $r$ if $J$ is a \C-model of $r$ and $I$ is a \C-model of $r^J$. The
set of all \SE-models of a rule $r$ is denoted by $\modse{r}$. The \SE-models
of a program $\prP$ are the \SE-models of all rules in $\prP$. A set of
\SE-interpretations $\mcS$ is called \textbf{rule-representable} if there
exists a rule $r$ such that $\mcS = \modse{r}$.

We say that a rule $r$ is \textbf{\SE-tautological} if $\modse{r} = \seint$.
Note that the canonical tautology $\varepsilon$ (c.f. Definition
\ref{def:canonical tautology}) is \SE-tautological. We say that two rules $r,
r'$ are \emph{strongly equivalent}, or \emph{\SE-equivalent}, if they have the
same set of \SE-models.

\section{Rule Equivalence Classes and their Canonical Rules}

\label{sec:se models of rules}

Our goal is to find useful insights into the inner workings of the \SE-model
semantics with respect to single rules. In order to do so, we first introduce
a set of representatives of rule equivalence classes induced by \SE-models and
show how the representative of a class can be constructed given one of its
members. Then we show how to reconstruct a representative from the set of its
\SE-models. Finally, we pinpoint the conditions under which a set of
\SE-interpretations is rule-representable.

\subsection{Canonical Rules}

We start by bringing out simple but powerful transformations that simplify a
given rule while preserving its \SE-models. Most of these results have already
been formulated in various ways \cite{Inoue1998,Inoue2004,Cabalar2007a}. The
following result summarises the conditions under which a rule is
\SE-tautological:

\begin{lemma}
	[Consequence of Theorem 4.4 in \cite{Inoue2004}; part i) of Lemma 2 in
	\cite{Cabalar2007a}]
	\label{lemma:se:tautology}
	Let $H$ and $B$ be sets of literals and $\at$ be an atom. Then a rule is
	\SE-tautological if it takes any of the following forms:
	\begin{align*}
		\at; H &\lpif \at, B. &
		H; \lpnot \at &\lpif B, \lpnot \at. &
		H \lpif B, \at, \lpnot \at.
	\end{align*}
\end{lemma}
\begin{proof}
	See Appendix \ref{app:proofs}, page \pageref{proof:lemma:se:tautology}. \qed
\end{proof}

\noindent Thus, repeating an atom in different ``components'' of the rule
frequently causes the rule to be \SE-tautological. In particular, this happens
if the same atom occurs in the positive head and positive body, or in the
negative head and negative body, or in the positive and negative bodies of a
rule. How about the cases when the head contains a negation of a literal from
the body? The following Lemma clarifies this situation:

\begin{lemma}
	[Consequence of (3) and (4) in Lemma 1 in \cite{Cabalar2007a}]
	\label{lemma:se:head repetition}
	Let $H$ and $B$ be sets of literals and $L$ be a literal. Then rules of the
	following forms are \SE-equivalent:
	\begin{align}
		H; \lpnot L &\lpif L, B. & H &\lpif L, B.
	\end{align}
\end{lemma}
\begin{proof}
	See Appendix \ref{app:proofs}, page \pageref{proof:lemma:se:head
	repetition}. \qed
\end{proof}

\noindent So if a literal is present in the body of a rule, its negation can
be removed from the head.

Until now we have seen that a rule $r$ that has a common atom in at
least two of the sets $H(r)^+ \cup H(r)^-$, $B(r)^+$ and $B(r)^-$ is either
\SE-tautological, or \SE-equivalent to a rule where the atom is omitted from
the rule's head. So such a rule is always \SE-equivalent either to the
canonical tautology $\varepsilon$, or to a rule without such repetitions.
Perhaps surprisingly, repetitions in positive and negative head cannot be
simplified away. For example, over the alphabet $\lang_\at = \set{\at}$, the
rule ``$\at; \lpnot \at \lpif.$'' has two \SE-models, $\an{\emptyset,
\emptyset}$ and $\an{\set{\at}, \set{\at}}$, so it is not \SE-tautological,
nor is it \SE-equivalent to any of the facts ``$\at.$'' and ``$\lpnot \at.$''.
Actually, it is not very difficult to see that it is not \SE-equivalent to
\emph{any} other rule, even over larger alphabets. So the fact that an atom is
in both $H(r)^+$ and $H(r)^-$ cannot all by itself imply that some kind of
\SE-models preserving rule simplification is possible.

The final Lemma reveals a special case in which we can eliminate the whole
negative head of a rule and move it to its positive body. This occurs whenever
the positive head is empty.

\begin{lemma}[Related to Corollary 4.10 in \cite{Inoue1998} and Corollary 1 in
	\cite{Cabalar2007a}] \label{lemma:se:constraint}
	Let $H^-$ be a set of negative literals, $B$ be a set of literals and $\at$
	be an atom. Then rules of the following forms are \SE-equivalent:
	\begin{align*}
		\lpnot \at; H^- &\lpif B. & H^- &\lpif \at, B.
	\end{align*}
\end{lemma}
\begin{proof}
	See Appendix \ref{app:proofs}, page \pageref{proof:lemma:se:constraint}.
	\qed
\end{proof}

Armed with the above results, we can introduce the notion of a canonical rule.
Each such rule represents a different equivalence class on the set of all
rules induced by the \SE-model semantics. In other words, every rule is
\SE-equivalent to exactly one canonical rule. After the definition, we provide
constructive transformations which show that this is indeed the case. Note
that the definition can be derived directly from the Lemmas above:

\begin{definition}[Canonical Rule]
	We say a rule $r$ is \emph{canonical} if either it is $\varepsilon$, or
	the following conditions are satisfied:
	\begin{enumerate}
		\item The sets $H(r)^+ \cup H(r)^-$, $B(r)^+$ and $B(r)^-$ are pairwise
			disjoint.
		\item If $H(r)^+$ is empty, then $H(r)^-$ is also empty.
	\end{enumerate}
\end{definition}

This definition is closely related with the notion of a \emph{fundamental
rule} introduced in Definition 1 of \cite{Cabalar2007a}. There are two
differences between canonical and fundamental rules: (1) a fundamental rule
must satisfy condition 1. above, but need not satisfy condition 2.; (2) no
\SE-tautological rule is fundamental. As a consequence, fundamental rules do
not cover all rule-representable sets of \SE-interpretations, and two distinct
fundamental rules may still be \SE-equivalent. From the point of view of rule
equivalence classes induced by \SE-model semantics, there is one class that
contains no fundamental rule, and some classes contain more than one
fundamental rule. In the following we show that canonical rules overcome both
of these limitations of fundamental rules. In other words, every rule is
\SE-equivalent to exactly one canonical rule. To this end, we define
constructive transformations that directly show the mutual relations between
rule syntax and semantics.

The following transformation provides a direct way of constructing a
canonical rule that is \SE-equivalent to a given rule $r$.

\begin{definition}[Transformation into a Canonical Rule]
	Given a rule $r$, by $\secan{r}$ we denote a canonical rule constructed as
	follows: If any of the sets $H(r)^+ \cap B(r)^+$, $H(r)^- \cap B(r)^-$ and
	$B(r)^+ \cap B(r)^-$ is nonempty, then $\secan{r}$ is $\varepsilon$.
	Otherwise, $\secan{r}$ is of the form $H^+; \lpnot H^- \lpif B^+, \lpnot
	B^-.$ where
	\begin{itemize}
		\item $H^+ = H(r)^+ \setminus B(r)^-$.
		\item If $H^+$ is empty, then $H^- = \emptyset$ and $B^+ = B(r)^+ \cup
			H(r)^-$.
		\item If $H^+$ is nonempty, then $H^- = H(r)^- \setminus B(r)^+$ and $B^+
			= B(r)^+$.
		\item $B^- = B(r)^-$.
	\end{itemize}
\end{definition}

\noindent Correctness of the transformation follows directly from Lemmas
\ref{lemma:se:tautology} to \ref{lemma:se:constraint}.

\begin{theorem} \label{thm:se:canonical equivalence}
	Every rule $r$ is \SE-equivalent to the canonical rule $\secan{r}$.
\end{theorem}
\begin{proof}
	See Appendix \ref{app:proofs}, page \pageref{proof:thm:se:canonical
	equivalence}. \qed
\end{proof}

What remains to be proven is that no two different canonical rules are
\SE-equivalent. In the next Subsection we show how every canonical rule
can be reconstructed from the set of its \SE-models. As a consequence, no two
different canonical rules can have the same set of \SE-models.

\subsection{Reconstructing Rules}

\label{subsec:induced rules}

In order to reconstruct a rule $r$ from the set $\mcS$ of its \SE-models, we
need to understand how exactly each literal in the rule influences its models.
The following Lemma provides a useful characterisation of the set of
countermodels of a rule in terms of syntax:

\begin{lemma}[Different formulation of Theorem 4 in \cite{Cabalar2007a}]
	\label{lemma:se:model conditions}
	Let $r$ be a rule. An \SE-interpretation $\an{I, J}$ is not an \SE-model of
	$r$ if and only if the following conditions are satisfied:
	\begin{enumerate}
		\item $H(r)^- \cup B(r)^+ \subseteq J$ and $J \subseteq \lang \setminus
			B(r)^-$.

		\item Either $J \subseteq \lang \setminus H(r)^+$ or both $B(r)^+
			\subseteq I$ and $I \subseteq \lang \setminus H(r)^+$.
	\end{enumerate}
\end{lemma}
\begin{proof}
	See Appendix \ref{app:proofs}, page \pageref{proof:lemma:se:model
	conditions}. \qed
\end{proof}

The first condition together with the first disjunct of the second condition
hold if and only if $J$ is not a \C-model of $r$. The second disjunct then
captures the case when $I$ is not a \C-model of $r^J$.

If we take a closer look at these conditions, we find that the presence of a
negative body atom in $J$ guarantees that the first condition is falsified, so
$\an{I, J}$ is a model of $r$, regardless of the content of $I$. Somewhat
similar is the situation with positive head atoms -- whenever such an atom is
present in $I$, it is also present in $J$, so the second condition is
falsified and $\an{I, J}$ is a model of $r$. Thus, if $\mcS$ is the set of
\SE-models of a rule $r$, then every atom $p \in B(r)^-$ satisfies
\begin{equation} \label{eq:lemma:se:negative body}
	p \in J \text{ implies } \an{I, J} \in \mcS \tag{$C_{B^-}$}
\end{equation}
and every atom $p \in H(r)^+$ satisfies
\begin{equation} \label{eq:lemma:se:positive head}
	p \in I \text{ implies } \an{I, J} \in \mcS \enspace. \tag{$C_{H^+}$}
\end{equation}
If we restrict ourselves to canonical rules different from $\varepsilon$, we
find that these conditions are not only necessary, but, when combined
properly, also sufficient to decide what atoms belong to the negative body and
positive head of the rule.

For the rest of this Subsection, we assume that $r$ is a canonical rule
different from $\varepsilon$ and $\mcS$ is the set of \SE-models of $r$.
Keeping in mind that every atom that satisfies condition
\eqref{eq:lemma:se:negative body} also satisfies condition
\eqref{eq:lemma:se:positive head} (because $I$ is a subset of $J$), and that
$B(r)^-$ is by definition disjoint from $H(r)^+$, we arrive at the following
results:

\begin{lemma} \label{lemma:se:negative body and positive head}
	An atom $\at$ belongs to $B(r)^-$ if and only if for all $\an{I, J} \in
	\seint$, the condition \eqref{eq:lemma:se:negative body} is satisfied.  An
	atom $\at$ belongs to $H(r)^+$ if and only if it does not belong to $B(r)^-$
	and for all $\an{I, J} \in \seint$, the condition
	\eqref{eq:lemma:se:positive head} is satisfied.
\end{lemma}
\begin{proof}
	See Appendix \ref{app:proofs}, page \pageref{proof:lemma:se:negative body
	and positive head}. \qed
\end{proof}

As can be seen from Lemma \ref{lemma:se:model conditions}, the role of
positive body and negative head atoms is dual to that of negative body and
positive head atoms. Intuitively, their absence in $J$, and sometimes also in
$I$, implies that $\an{I, J}$ is an \SE-model of $r$. It follows from the
first condition of Lemma \ref{lemma:se:model conditions} that if $p$ belongs
to $H(r)^- \cup B(r)^+$, then the following condition is satisfied:
\begin{equation}
	p \notin J \text{ implies } \an{I, J} \in \mcS \enspace.
		\label{eq:lemma:se:negative head}  \tag{$C_{H^-}$}
\end{equation}
Furthermore, the second condition in Lemma \ref{lemma:se:model conditions}
implies that every $p \in B(r)^+$ satisfies the following condition:
\begin{equation}
	p \notin I \text{ and } J \cap H(r)^+ \neq \emptyset 
		\text{ implies } \an{I, J} \in \mcS \enspace.
		\label{eq:lemma:se:positive body} \tag{$C_{B^+}$}
\end{equation}
These observations lead to the following results:

\begin{lemma} \label{lemma:se:positive body and negative head}
	An atom $\at$ belongs to $B(r)^+$ if and only if for all $\an{I, J} \in
	\seint$, the conditions \eqref{eq:lemma:se:negative head} and
	\eqref{eq:lemma:se:positive body} are satisfied.  An atom $\at$ belongs to
	$H(r)^-$ if and only if it does not belong to $B(r)^+$ and for all $\an{I,
	J} \in \seint$, the condition \eqref{eq:lemma:se:negative head} is
	satisfied.
\end{lemma}
\begin{proof}
	See Appendix \ref{app:proofs}, page \pageref{proof:lemma:se:positive body
	and negative head}. \qed
\end{proof}

Together, the two Lemmas above are sufficient to reconstruct a canonical rule
from its set of \SE-models. The following definition sums up these results by
introducing the notion of a rule induced by a set of \SE-interpretations:

\begin{definition}[Rule Induced by a Set of \SE-Interpretations]
	\label{def:se:models to rule}

	Let $\mcS$ be a set of \SE-interpretations.

	An atom $\at$ is called an \emph{\SNegBody{} atom} if every
	\SE-interpretation $\an{I, J}$ with $\at \in J$ belongs to $\mcS$.  An atom
	$\at$ is called an \emph{\SPosHead{} atom} if it is not an \SNegBody{} atom
	and every \SE-interpretation $\an{I, J}$ with $\at \in I$ belongs to $\mcS$.

	An atom $\at$ is called an \emph{\SPosBody{} atom} if every
	\SE-interpretation $\an{I, J}$ with $\at \notin J$ belongs to $\mcS$, and
	every \SE-interpretation $\an{I, J}$ with $\at \notin I$ and $J$ containing
	some \SPosHead{} atom also belongs to $\mcS$.  An atom $\at$ is called an
	\emph{\SNegHead{} atom} if it is not an \SPosBody{} atom and every
	\SE-interpretation $\an{I, J}$ with $\at \notin J$ belongs to $\mcS$.

	The sets of all \SNegBody{}, \SPosHead{}, \SPosBody{} and \SNegHead{} atoms
	are denoted by $B(\mcS)^-$, $H(\mcS)^+$, $B(\mcS)^+$ and $H(\mcS)^-$,
	respectively.  The \emph{rule induced by $\mcS$}, denoted by $\synt{\mcS}$,
	is defined as follows: If $\mcS = \seint$, then $\synt{\mcS}$ is
	$\varepsilon$; otherwise, $\synt{\mcS}$ is of the form
	\[
		H(\mcS)^+; \lpnot H(\mcS)^- \lpif B(\mcS)^+, \lpnot B(\mcS)^-.
	\]
\end{definition}

The main property of induced rules is that every canonical rule is induced by
its own set of \SE-models and can thus be ``reconstructed'' from its set of
\SE-models. This follows directly from Definition \ref{def:se:models to rule}
and Lemmas \ref{lemma:se:negative body and positive head} and
\ref{lemma:se:positive body and negative head}.

\begin{theorem} \label{thm:se:canonical from models}
	For every canonical rule $r$, $\synt{\modse{r}} = r$.
\end{theorem}
\begin{proof}
	See Appendix \ref{app:proofs}, page \pageref{proof:thm:se:canonical from
	models}. \qed
\end{proof}

This result, together with Theorem \ref{thm:se:canonical equivalence}, has a
number of consequences. First, for any rule $r$, the canonical rule
$\secan{r}$ is induced by the set of \SE-models of $r$.

\begin{corollary}
	For every rule $r$, $\synt{\modse{r}} = \secan{r}$.
\end{corollary}
\begin{proof}
	Follows directly from Theorem \ref{thm:se:canonical equivalence} and Theorem
	\ref{thm:se:canonical from models}. \qed
\end{proof}

Furthermore, Theorem \ref{thm:se:canonical from models} directly implies that
for two different canonical rules $r_1, r_2$ we have $\synt{\modse{r_1}} =
r_1$ and $\synt{\modse{r_2}} = r_2$, so $\modse{r_1}$ and $\modse{r_2}$ must
differ.

\begin{corollary} \label{cor:se:canonical not equivalent}
	No two different canonical rules are \SE-equivalent.
\end{corollary}
\begin{proof}
	Follows directly from the Theorem \ref{thm:se:canonical from models}. \qed
\end{proof}

Finally, the previous Corollary together with Theorem \ref{thm:se:canonical
equivalence} imply that for every rule there not only exists an \SE-equivalent
canonical rule, but this rule is also unique.

\begin{corollary}
	Every rule is \SE-equivalent to exactly one canonical rule.
\end{corollary}
\begin{proof}
	Follows directly from Theorem \ref{thm:se:canonical equivalence} and
	Corollary \ref{cor:se:canonical not equivalent}. \qed
\end{proof}

\subsection{Sets of \SE-Interpretations Representable by a Rule}

\label{sec:se models of programs}

Naturally, not all sets of \SE-interpretations correspond to a single rule,
otherwise any program could be reduced to a single rule. The conditions under
which a set of \SE-interpretations is rule-representable are worth examining.

A set of \SE-models $\mcS$ of a program is always \emph{well-defined}, i.e.
whenever $\mcS$ contains $\an{I, J}$, it also contains $\an{J, J}$. Moreover,
for every well-defined set of \SE-interpretations $\mcS$ there exists a
program $\prP$ such that $\mcS = \modse{\prP}$ \cite{Delgrande2008}.

We offer two approaches to find a similar condition for the class of
rule-representable sets of \SE-interpretations. The first is based on induced
rules defined in the previous Subsection, while the second is formulated using
lattice theory and is a consequence of Lemma \ref{lemma:se:model conditions}.

The first characterisation follows from two properties of the $\synt{\cdot}$
transformation. First, it can be applied to any set of \SE-interpretations,
even those that are not rule-representable. Second, if $\synt{\mcS} =
r$, then it holds that $\modse{r}$ is a subset of $\mcS$.

\begin{lemma} \label{lemma:se:models induced least}
	The set of all \SE-models of a canonical rule $r$ is the least among
	all sets of \SE-interpretations $\mcS$ such that $\synt{\mcS} = r$.
\end{lemma}
\begin{proof}
	See Appendix \ref{app:proofs}, page \pageref{proof:lemma:se:models induced
	least}. \qed
\end{proof}

\noindent Thus, to verify that $\mcS$ is rule-representable, it suffices to
check that all interpretations from $\mcS$ are models of $\synt{\mcS}$.

The second characterisation follows from Lemma \ref{lemma:se:model conditions}
which tells us that if $\mcS$ is rule-representable, then its complement
consists of \SE-interpretations $\an{I, J}$ following a certain pattern. Their
second component $J$ always contains a fixed set of atoms and is itself
contained in another fixed set of atoms. Their first component $I$ satisfies a
similar property, but only if a certain further condition is satisfied by $J$.
More formally, for the sets
\begin{align*}
	I^\bot &= B(r)^+, & I^\top &= \lang \setminus H(r)^+, &
	J^\bot &= H(r)^- \cup B(r)^+, & J^\top &= \lang \setminus B(r)^-,
\end{align*}
it holds that all \SE-interpretations from the complement of $\mcS$ are of the
form $\an{I, J}$ where $J^\bot \subseteq J \subseteq J^\top$ and either $J
\subseteq I^\top$ or $I^\bot \subseteq I \subseteq I^\top$. It turns out that
this also holds vice versa: if the complement of $\mcS$ satisfies the above
property, then $\mcS$ is rule-representable. Furthermore, to accentuate the
particular structure that arises, we can substitute the condition $J^\bot
\subseteq J \subseteq J^\top$ with saying that $J$ belongs to a convex
sublattice of $\pint$.\footnote{A sublattice $L$ of $L'$ is \emph{convex} if
$c \in L$ whenever $a, b \in L$ and $a \leq c \leq b$ holds in $L'$. For more
details see e.g. \cite{Davey1990}.} A similar substitution can be performed
for $I$, yielding:

\begin{theorem} \label{thm:se:rule representable}
	Let $\mcS$ be a set of \SE-interpretations. Then the following conditions
	are equivalent:
	\begin{enumerate}
		\item The set of \SE-interpretations $\mcS$ is rule-representable.
		\item All \SE-interpretations from $\mcS$ are \SE-models of $\synt{\mcS}$.
		\item There exist convex sublattices $L_1, L_2$ of $\an{\pint, \subseteq}$
			such that the complement of $\mcS$ relative to $\seint$ is equal to
			\[
			\Set{\an{I, J} \in \seint | I \in L_1 \land J \in L_2} \cup
				\Set{\an{I, J} \in \seint | J \in L_1 \cap L_2} \enspace.
			\]
	\end{enumerate}
\end{theorem}
\begin{proof}
	See Appendix \ref{app:proofs}, page \pageref{proof:thm:se:rule
	representable}. \qed
\end{proof}

\section{Discussion}

\label{sec:discussion}

The presented results mainly serve to facilitate the transition back and forth
between a rule and the set of its \SE-models. They also make it possible to
identify when a given set of \SE-models is representable by a single rule. We
believe that in situations where information on literal dependencies,
expressed in individual rules, is essential for defining operations on logic
programs, the advantages of dealing with rules on the level of semantics
instead of on the level of syntax are significant. The semantic view takes
care of stripping away unnecessary details and since the introduced notions
and operators are defined in terms of semantic objects, it should be much
easier to introduce and prove their semantic properties.

These results can be used for example in the context of program updates to
define an update semantics based on the \emph{rule rejection principle}
\cite{Alferes2000} and operating on \emph{sets of sets of \SE-models}. Such a
semantics can serve as a bridge between syntax-based approaches to rule
updates, and the principles and semantic distance measures known from the area
of Belief Change. The next steps towards such a semantics involve a definition
of the notion of support for a literal by a set of \SE-models (of a rule).
Such a notion can then foster a better understanding of desirable properties
for semantic rule update operators.

On a different note, viewing a logic program as the \emph{set of sets of
\SE-models} of rules inside it leads naturally to the introduction of the
following new notion of program equivalence:

\begin{definition}[Strong Rule Equivalence]
	Programs $\prP_1, \prP_2$ are \emph{\SR-equivalent}, denoted by $\prP_1
	\equiv_{\mSR} \prP_2$, if
	\[
	\Set{ \modse{r} | r \in \prP_1 \cup \set{\varepsilon} }
		= \Set{ \modse{r} | r \in \prP_2 \cup \set{\varepsilon} } \enspace.
	\]
\end{definition}

Thus, two programs are \SR-equivalent if they contain the same rules, modulo
the \SE-model semantics. We add $\varepsilon$ to each of the two programs in
the definition so that presence or absence of tautological rules in a program
does not influence program equivalence. \SR-equivalence is stronger than
strong equivalence, in the following sense:

\begin{definition}[Strength of Program Equivalence]
	Let $\equiv_1, \equiv_2$ be equivalence relations on the set of all
	programs. We say that \emph{$\equiv_1$ is at least as strong as $\equiv_2$},
	denoted by $\equiv_1\,\succeq\,\equiv_2$, if $\prP_1 \equiv_1 \prP_2$
	implies $\prP_1 \equiv_2 \prP_2$ for all programs $\prP_1, \prP_2$. We say
	that \emph{$\equiv_1$ is stronger than $\equiv_2$}, denoted by
	$\equiv_1\,\succ\,\equiv_2$, if $\equiv_1\,\succeq\,\equiv_2$ but not
	$\equiv_2\,\succeq\,\equiv_1$.
\end{definition}

Thus, using the notation of the above definition, we can write $\equiv_{\mSR}
\,\succ\, \equiv_{\mSt}$, where $\equiv_{\mSt}$ denotes the relation of strong
equivalence. An example of programs that are strongly equivalent, but not
\SR-equivalent is $\prP = \set{p., q.}$ and $\prQ = \set{p., q \lpif p.}$,
which in many cases need to be distinguished from one another. We believe that
this notion of program equivalence is much more suitable for cases when the
dependency information contained in a program is of importance.

In certain cases, however, \SR-equivalence may be too strong. For instance, it
may be desirable to treat programs such as $\prP_1 = \set{p \lpif q.}$ and
$\prP_2 = \set{p \lpif q., p \lpif q, r.}$ in the same way because the extra
rule in $\prP_2$ is just a weakened version of the rule in $\prP_1$. For
instance, the notion of \emph{update equivalence} introduced in
\cite{Leite2003}, which is based on a particular approach to logic program
updates, considers programs $\prP_1$ and $\prP_2$ as equivalent because the
extra rule in $\prP_2$ cannot influence the result of any subsequent updates.
Since these programs are not \SR-equivalent, we also introduce the following
notion of program equivalence, which in terms of strength falls between strong
equivalence and \SR-equivalence.

\begin{definition}[Strong Minimal Rule Equivalence]
	Programs $\prP_1, \prP_2$ are \emph{\SMR-equivalent}, denoted by $\prP_1
	\equiv_{\mSMR} \prP_2$, if
	\[
		\min \Set{ \modse{r} | r \in \prP_1 \cup \set{\varepsilon} }
			= \min \Set{ \modse{r} | r \in \prP_2 \cup \set{\varepsilon} } \enspace,
	\]
	where $\min \mcS$ denotes the set of subset-minimal elements of $\mcS$.
\end{definition}

In order for programs to be \SMR-equivalent, they need not contain exactly the
same rules (modulo strong equivalence), it suffices if rules with subset-minimal
sets of \SE-models are the same (again, modulo strong equivalence). Certain
programs, such as $\prP_1$ and $\prP_2$ above, are not \SR-equivalent but they
are still \SMR-equivalent.

Related to this is the very strong notion of equivalence which was
introduced in \cite{Inoue2004}:

\begin{definition}[Strong Update Equivalence, c.f. Definition 4.1 in
	\cite{Inoue2004}]
	Two programs $\prP_1$, $\prP_1$ are \SU-equivalent, denoted by $\prP_1
	\equiv_{\mSU} \prP_2$, if for any programs $\prQ$, $\pr[R]$ it holds that
	the program $((\prP_1 \setminus \prQ) \cup \pr[R])$ has the same answer sets
	as the program $((\prP_2 \setminus \prQ) \cup \pr[R])$.
\end{definition}

Two programs are strongly update equivalent only under very strict conditions
-- it is shown in \cite{Inoue2004} that two programs are \SU-equivalent if and
only if their symmetric difference contains only \SE-tautological rules. This
means that programs such as $\prQ_1 = \set{\lpnot p.}$, $\prQ_2 = \set{\lpif
p.}$ and $\prQ_3 = \set{\lpnot p \lpif p.}$ are considered to be mutually
non-equivalent, even though the rules they contain are mutually
\SE-equivalent. This may be seen as too sensitive to rule syntax.

The following result formally establishes the relations between the discussed
notions of program equivalence:

\begin{theorem} \label{thm:se:equivalence comparison}
	\SU-equivalence is stronger than \SR-equivalence, which itself is stronger
	than \SMR-equivalence, which in turn is stronger than strong equivalence.
	That is,
	\[
		\equiv_{\mSU} \,\succ\, \equiv_{\mSR}
			\,\succ\, \equiv_{\mSMR} \,\succ\, \equiv_{\mSt} \enspace.
	\]
\end{theorem}
\begin{proof}
	See Appendix \ref{app:proofs}, page \pageref{proof:thm:se:equivalence
	comparison}. \qed
\end{proof}

The other notion of program equivalence introduced in \cite{Inoue2004},
\emph{strong update equivalence on common rules}, or \SUC-equivalence, is
incomparable in terms of strength to our new notions of equivalence. On the
one hand, \SR- and \SMR-equivalent programs such as $\set{\lpnot p.}$ and
$\set{\lpnot p., \lpif p.}$ are not \SUC-equivalent. On the other hand,
programs such as $\set{p., q \lpif p.}$ and $\set{q., p \lpif q.}$ are neither
\SR- nor \SMR-equivalent, but they are \SUC-equivalent. We believe that both
of these examples are more appropriately treated by the new notions of
equivalence.

The introduction of canonical rules, which form a set of representatives of
rule equivalence classes induced by \SE-models, also reveals the exact
expressivity of \SE-model semantics with respect to a single rule. From their
definition we can see that \SE-models are capable of distinguishing between
any pair of rules, except for (1) a pair of rules that only differ in the
number of repetitions of literals in their heads and bodies; (2) an integrity
constraint and a rule whose head only contains negative literals. We believe
that in the former case, there is little reason to distinguish between such
rules and so the transition from rules to their \SE-models has the positive
effect of stripping away of unnecessary details. However, the latter case has
more serious consequences. Although rules such as
\begin{alignat*}{2}
	\lpnot p &\lpif q. & \qquad \text{ and } \qquad &\lpif p, q.
\end{alignat*}
are usually considered to carry the same meaning, some existing work suggests
that they should be treated differently -- while the former rule gives a
reason for atom $p$ to become false whenever $q$ is true, the latter rule
simply states that the two atoms cannot be true at the same time, without
specifying a way to resolve this situation if it were to arise
\cite{Alferes2000,Alferes2005}. If we view a rule through the set of its
\SE-models, we cannot distinguish these two kinds of rules anymore. Whenever
this is important, either \emph{strong update equivalence} is used, which is
perhaps \emph{too} sensitive to the syntax of rules, or a new characterisation
of Answer-Set Programming needs to be discovered, namely one that is not based
on the logic of Here-and-There \cite{Lukasiewicz1941,Pearce1997}.

\section*{Acknowledgement}

We would like to thank Han The Anh, Matthias Knorr and the anonymous reviewers
for their comments that helped to improve the paper. Martin Slota is supported
by FCT scholarship SFRH~/~BD~/~38214~/~2007.

\bibliographystyle{unsrt}
\bibliography{bibliography}

\begin{thebibliography}{10}

\bibitem{Lifschitz2001}
Vladimir Lifschitz, David Pearce, and Agust{\'i}n Valverde.
\newblock Strongly equivalent logic programs.
\newblock {\em ACM Transactions on Computational Logic}, 2(4):526--541, 2001.

\bibitem{Inoue2004}
Katsumi Inoue and Chiaki Sakama.
\newblock Equivalence of logic programs under updates.
\newblock In Jos{\'e}~J{\'u}lio Alferes and Jo{\~a}o~Alexandre Leite, editors,
  {\em Proceedings of the 9th European Conference on Logics in Artificial
  Intelligence}, volume 3229 of {\em Lecture Notes in Computer Science}, pages
  174--186, Lisbon, Portugal, September 27-30 2004. Springer.

\bibitem{Damasio1997}
Carlos~Viegas Dam{\'a}sio, Lu\'{\i}s~Moniz Pereira, and Michael Schroeder.
\newblock {REVISE}: Logic programming and diagnosis.
\newblock In J{\"u}rgen Dix, Ulrich Furbach, and Anil Nerode, editors, {\em
  Proceedings of the 4th International Conference on Logic Programming and
  Nonmonotonic Reasoning}, volume 1265 of {\em Lecture Notes in Computer
  Science}, pages 354--363, Dagstuhl Castle, Germany, July 28-31 1997.
  Springer.

\bibitem{Alferes2000}
Jos{\'e}~J{\'u}lio Alferes, Jo{\~a}o~Alexandre Leite, Lu{\'i}s~Moniz Pereira,
  Halina Przymusinska, and Teodor~C. Przymusinski.
\newblock Dynamic updates of non-monotonic knowledge bases.
\newblock {\em The Journal of Logic Programming}, 45(1-3):43--70,
  September/October 2000.

\bibitem{Eiter2002}
Thomas Eiter, Michael Fink, Giuliana Sabbatini, and Hans Tompits.
\newblock On properties of update sequences based on causal rejection.
\newblock {\em Theory and Practice of Logic Programming}, 2(6):721--777, 2002.

\bibitem{Sakama2003}
Chiaki Sakama and Katsumi Inoue.
\newblock An abductive framework for computing knowledge base updates.
\newblock {\em Theory and Practice of Logic Programming}, 3(6):671–713, 2003.

\bibitem{Zhang2006}
Yan Zhang.
\newblock Logic program-based updates.
\newblock {\em ACM Transactions on Computational Logic}, 7(3):421--472, 2006.

\bibitem{Alferes2005}
Jos{\'e}~J{\'u}lio Alferes, Federico Banti, Antonio Brogi, and
  Jo{\~a}o~Alexandre Leite.
\newblock The refined extension principle for semantics of dynamic logic
  programming.
\newblock {\em Studia Logica}, 79(1):7--32, 2005.

\bibitem{Delgrande2007}
James~P. Delgrande, Torsten Schaub, and Hans Tompits.
\newblock A preference-based framework for updating logic programs.
\newblock In Chitta Baral, Gerhard Brewka, and John~S. Schlipf, editors, {\em
  Proceedings of the 9th International Conference on Logic Programming and
  Nonmonotonic Reasoning}, volume 4483 of {\em Lecture Notes in Computer
  Science}, pages 71--83, Tempe, AZ, USA, May 15-17 2007. Springer.

\bibitem{Delgrande2008}
James~P. Delgrande, Torsten Schaub, Hans Tompits, and Stefan Woltran.
\newblock Belief revision of logic programs under answer set semantics.
\newblock In Gerhard Brewka and J{\'e}r{\^o}me Lang, editors, {\em Proceedings
  of the 11th International Conference on Principles of Knowledge
  Representation and Reasoning}, pages 411--421, Sydney, Australia, September
  16-19 2008. AAAI Press.

\bibitem{Delgrande2010}
James~P. Delgrande.
\newblock A {P}rogram-{L}evel {A}pproach to {R}evising {L}ogic {P}rograms under
  the {A}nswer {S}et {S}emantics.
\newblock {\em Theory and Practice of Logic Programming, 26th Int'l. Conference
  on Logic Programming Special Issue}, 10(4-6):565--580, July 2010.

\bibitem{Gardenfors1992}
Peter G{\"a}rdenfors.
\newblock {\em Belief Revision}, chapter Belief Revision: An Introduction,
  pages 1--28.
\newblock Cambridge University Press, 1992.

\bibitem{Slota2010b}
Martin Slota and Jo{\~a}o Leite.
\newblock On semantic update operators for answer-set programs.
\newblock In Helder Coelho, Rudi Studer, and Michael Wooldridge, editors, {\em
  Proceedings of the 19th European Conference on Artificial Intelligence},
  volume 215 of {\em Frontiers in Artificial Intelligence and Applications},
  pages 957--962, Lisbon, Portugal, August 16-20 2010. IOS Press.

\bibitem{Apt1988}
Krzysztof~R. Apt, Howard~A. Blair, and Adrian Walker.
\newblock Towards a theory of declarative knowledge.
\newblock In {\em Foundations of Deductive Databases and Logic Programming},
  pages 89--148. Morgan Kaufmann, 1988.

\bibitem{Dix1995a}
J{\"u}rgen Dix.
\newblock A classification theory of semantics of normal logic programs: {II}.
  {W}eak properties.
\newblock {\em Fundamenta Informaticae}, 22(3):257--288, 1995.

\bibitem{Lukasiewicz1941}
Jan {\L}ukasiewicz.
\newblock Die {L}ogik und das {G}rundlagenproblem.
\newblock In {\em {L}es Entretiens de Z{\"u}rich sue les Fondements et la
  m{\'e}thode des sciences math{\'e}matiques 1938}, pages 82--100. Z{\"u}rich,
  1941.

\bibitem{Pearce1997}
David Pearce.
\newblock A new logical characterisation of stable models and answer sets.
\newblock In J{\"u}rgen Dix, Lu\'{\i}s~Moniz Pereira, and Teodor~C.
  Przymusinski, editors, {\em Proceedings of the 6th Workshop on Non-Monotonic
  Extensions of Logic Programming}, volume 1216 of {\em Lecture Notes in
  Computer Science}, pages 57--70, Bad Honnef, Germany, September 5-6 1997.
  Springer.

\bibitem{Turner2003}
Hudson Turner.
\newblock Strong equivalence made easy: nested expressions and weight
  constraints.
\newblock {\em Theory and Practice of Logic Programming}, 3(4-5):609--622,
  2003.

\bibitem{Gelfond1988}
Michael Gelfond and Vladimir Lifschitz.
\newblock The stable model semantics for logic programming.
\newblock In Robert~A. Kowalski and Kenneth~A. Bowen, editors, {\em Proceedings
  of the 5th International Conference and Symposium on Logic Programming},
  pages 1070--1080, Seattle, Washington, August 15-19 1988. MIT Press.

\bibitem{Inoue1998}
Katsumi Inoue and Chiaki Sakama.
\newblock Negation as failure in the head.
\newblock {\em Journal of Logic Programming}, 35(1):39--78, 1998.

\bibitem{Cabalar2007a}
Pedro Cabalar, David Pearce, and Agust{\'i}n Valverde.
\newblock Minimal logic programs.
\newblock In Ver{\'o}nica Dahl and Ilkka Niemel{\"a}, editors, {\em Proceedings
  of the 23rd International Conference on Logic Programming (ICLP 2007)},
  volume 4670 of {\em Lecture Notes in Computer Science}, pages 104--118,
  Porto, Portugal, September 8-13 2007. Springer.

\bibitem{Davey1990}
Brian~A. Davey and Hilary~A. Priestley.
\newblock {\em Introduction to Lattices and Order}.
\newblock Cambridge University Press, 1990.

\bibitem{Leite2003}
Jo{\~a}o~Alexandre Leite.
\newblock {\em Evolving Knowledge Bases}, volume~81 of {\em Frontiers of
  Artificial Intelligence and Applications, xviii + 307 p. Hardcover}.
\newblock IOS Press, 2003.

\end{thebibliography}

\begin{extended}

\newpage

\appendix

\section{Proofs} \label{app:proofs}

\begin{lemma*}{lemma:se:tautology}
	Let $H$ and $B$ be sets of literals and $\at$ be an atom. Then a rule is
	\SE-tautological if it takes any of the following forms:
	\begin{align*}
		\at; H &\lpif \at, B. &
		H; \lpnot \at &\lpif B, \lpnot \at. &
		H \lpif B, \at, \lpnot \at.
	\end{align*}
\end{lemma*}
\begin{proof}
	\label{proof:lemma:se:tautology}

	First assume that rule $r$ is of the first form. We need to show that any
	\SE-interpretation is an \SE-model of $r$. Suppose $\an{I, J}$ is some
	\SE-interpretation. Rule $r$ is \C-tautological, so $J$ is a \C-model of
	$r$. Furthermore, $r^J$ is either $\varepsilon$, or it inherits $\at$ in
	both its head and body from $r$. In any case, $r$ is \C-tautological, so $I$
	is a \C-model of $r^J$. Consequently, $\an{I, J}$ is an \SE-model of $r$.

	Now suppose $r$ is of the second form. As before, given an
	\SE-interpretation $\an{I, J}$, we see that $J$ is a \C-model of $r$ because
	$r$ is \C-tautological. Furthermore, $r^J$ will necessarily end up being
	equal to $\varepsilon$ because of the atom $\at$ common to $H(r)^-$ and
	$B(r)^-$, regardless of how $J$ interprets $\at$. So $I$ is a \C-model of
	$r^J$, and, consequently, $\an{I, J}$ is an \SE-model of $r$.

	Finally, suppose $r$ takes the third form and take an \SE-interpretation
	$\an{I, J}$. Rule $r$ can again easily be verified to be \C-tautological, so
	$J$ is a \C-model of $r$. If all atoms from $B(r)^-$ are false under $J$,
	then $r^J$ contains the atom $\at$ in its body that is false under $J$, thus
	also false under $I$ since $I$ is a subset of $J$.  Consequently, $I$ is a
	\C-model of $r^J$. On the other hand, if at least one of atoms from $B(r)^-$
	is true under $J$, then $r^J$ is equal to $\varepsilon$, so again, $I$ is a
	\C-model of $r^J$. Consequently, $\an{I, J}$ is an \SE-model of $r$. \qed
\end{proof}

\begin{lemma} \label{lemma:se:positive head repetition}
	Let $H$ and $B$ be sets of literals and $\at$ be an atom. Then rules of the
	following forms are \SE-equivalent:
	\begin{align*}
		\at; H &\lpif B, \lpnot \at. & H &\lpif B, \lpnot \at.
	\end{align*}
\end{lemma}
\begin{proof}
	Let the first rule be denoted by $r_1$ and the second by $r_2$ and suppose
	$\an{I, J}$ is an \SE-interpretation. We will show that $\an{I, J}$ is an
	\SE-model of $r_1$ if and only if it is an \SE-model of $r_2$. We can easily
	see that rules $r_1$, $r_2$ are \C-equivalent. So $J$ is either not a
	\C-model of any of them or it is a \C-model of both of them. In the former
	case, $\an{I, J}$ is not an \SE-model of any of the two rules and we are
	finished. In the latter case, we need to distinguish two cases:
	\begin{enumerate}
		\renewcommand{\labelenumi}{\alph{enumi})}

		\item If $J$ is a not a model of the bodies of $r_1$ and $r_2$ (the bodies
			are identical), then either $r_1^J$ and $r_2^J$ are equal to
			$\varepsilon$, or they contain an atom in their bodies that is false
			under $J$, thus also false under $I$ because $I$ is a subset of $J$. In
			any case, $I$ is a \C-model of both $r_1^J$ and $r_2^J$, so $\an{I, J}$
			is an \SE-model of both $r_1$ and $r_2$.

		\item If $J$ is a model of the bodies of $r_1$ and $r_2$, then $\at$ is
			false under $J$, and, since $J$ is a \C-model of $r_1$, some literal
			from $H$ must be true under $J$. Consequently, either both $r_1^J$ and
			$r_2^J$ are equal to $\varepsilon$ and $I$ is a \C-model of both of
			them, or $r_1^J$ and $r_2^J$ only differ in the single extra atom $\at$
			that $r_1^J$ has in the head. However, since $I$ is a subset of $J$ and
			$\at$ is false under $J$, $\at$ cannot be true under $I$, so $I$ is
			either not a \C-model of any of $r_1^J$, $r_2^J$, or it is a \C-model of
			both of them. In any case, $\an{I, J}$ is a \SE-model of $r_1$ if and
			only if it is an \SE-model of $r_2$.
	\end{enumerate}
	Thus, we have proven that every \SE-interpretation is an \SE-model of $r_1$
	if and only if it is an \SE-model of $r_2$. In other words, $r_1$ and $r_2$
	are \SE-equivalent. \qed
\end{proof}

\begin{lemma} \label{lemma:se:negative head repetition}
	Let $H$ and $B$ be sets of literals and $\at$ be an atom. Then rules of the
	following forms are \SE-equivalent:
	\begin{align*}
		H; \lpnot \at &\lpif \at, B. & H &\lpif \at, B.
	\end{align*}
\end{lemma}
\begin{proof}
	\label{proof:lemma:se:negative head repetition}

	Let the first rule be denoted by $r_1$ and the second by $r_2$. Suppose
	$\an{I, J}$ is some \SE-interpretation. If $\at$ is false under $J$, then
	$J$ is a \C-model of both $r_1$ and $r_2$. Furthermore, $r_1^J$ is equal to
	$\varepsilon$ and $r_2^J$ is either equal to $\varepsilon$ or its body
	contains $\at$, and so is not true under $I$. In any case, $I$ is a \C-model
	of both $r_1^J$ and $r_2^J$, so $\an{I, J}$ is an \SE-model of both $r_1$
	and $r_2$.

	On the other hand, if $\at$ is true under $J$, then $J$ is a \C-model of
	$r_1$ if and only if it is a \C-model of $r_2$ because the extra literal
	$\lpnot a$ in the head of $r_1$ cannot be satisfied. Also, $r_1^J$ is
	identical to $r_2^J$, so $I$ is a \C-model of $r_1^J$ if and only if it is a
	\C-model of $r_2^J$.  Consequently, $\an{I, J}$ is an \SE-model of $r_1$ if
	and only if it is an \SE-model of $r_2$.

	Thus, we have proven that every \SE-interpretation is an \SE-model of $r_1$
	if and only if it is an \SE-model of $r_2$. In other words, $r_1$ and $r_2$
	are \SE-equivalent. \qed
\end{proof}

\begin{lemma*}{lemma:se:head repetition}
	Let $H$ and $B$ be sets of literals and $L$ be a literal. Then rules of the
	following forms are \SE-equivalent:
	\begin{align}
		H; \lpnot L &\lpif L, B. & H &\lpif L, B.
	\end{align}
\end{lemma*}
\begin{proof}
	\label{proof:lemma:se:head repetition}
	Follows from Lemmas \ref{lemma:se:positive head repetition} and
	\ref{lemma:se:negative head repetition}. \qed
\end{proof}

\begin{lemma*}{lemma:se:constraint}
	Let $H^-$ be a set of negative literals, $B$ be a set of literals and $\at$
	be an atom. Then rules of the following forms are \SE-equivalent:
	\begin{align*}
		\lpnot \at; H^- &\lpif B. & H^- &\lpif \at, B.
	\end{align*}
\end{lemma*}
\begin{proof}
	\label{proof:lemma:se:constraint}

	Let the first rule be denoted by $r_1$ and the second by $r_2$. Suppose
	$\an{I, J}$ is some \SE-interpretation. If $\at$ is false under $J$, then
	$J$ is a \C-model of both $r_1$ and $r_2$. Furthermore, $r_1^J$ is equal to
	$\varepsilon$ and the body of $r_2^J$ cannot be satisfied by $I$ because $I$
	is a subset of $J$ and $\at$ is not in $J$. Thus, $I$ is a \C-model of both
	$r_1^J$ and $r_2^J$, and $\an{I, J}$ is an \SE-model of both $r_1$ and
	$r_2$.

	On the other hand, if $\at$ is true under $J$, then we need to consider two
	cases:
	\begin{enumerate}
		\renewcommand{\labelenumi}{\alph{enumi})}
		\item If $J$ is not a \C-model of $r_1$, then it also cannot be a \C-model
			of $r_2$ because, as can easily be verified, $r_1$ is \C-equivalent to
			$r_2$. Hence, $\an{I, J}$ is an \SE-model of neither $r_1$ nor $r_2$.

		\item If $J$ is a \C-model of $r_1$, then it must also be a \C-model of
			$r_2$ because $r_1$ is \C-equivalent to $r_2$. Furthermore, $r_1^J$ and
			$r_2^J$ are either both equal to $\varepsilon$, or their heads are empty
			and $r_2^J$ has the extra atom $\at$ in the body. In the latter case, if
			$I$ were a model of the body of $r_1^J$, then $J$ would be a model of
			the body of $r_1$ but not of its head (which contains only negative
			literals), which contradicts the assumption that $J$ is a model of
			$r_1$. Thus, $I$ is not a model of the body of $r_1^J$, so it cannot be
			a model of the body of $r_2^J$ either. So $I$ is a \C-model of both
			$r_1^J$ and $r_2^J$ and $\an{I, J}$ is an \SE-model of both $r_1$ and
			$r_2$.
	\end{enumerate}
	Thus, we have proven that every \SE-interpretation is an \SE-model of $r_1$
	if and only if it is an \SE-model of $r_2$. In other words, $r_1$ and $r_2$
	are \SE-equivalent. \qed
\end{proof}

\begin{theorem*}{thm:se:canonical equivalence}
	Every rule $r$ is \SE-equivalent to the canonical rule $\secan{r}$.
\end{theorem*}
\begin{proof}
	\label{proof:thm:se:canonical equivalence}

	This can be shown by a careful iterative application of Lemmas
	\ref{lemma:se:tautology} to \ref{lemma:se:constraint}. First observe that if
	$\secan{r}$ is equal to $\varepsilon$, then by Lemma
	\ref{lemma:se:tautology} the rule $r$ is indeed \SE-equivalent to
	$\varepsilon$.

	In the principal case we can use Lemma \ref{lemma:se:positive head
	repetition} on all atoms shared between the positive head and negative body
	of $r$ and remove them one by one from the positive head of $r$ while
	preserving \SE-equivalence.  Similar situation occurs with atoms shared
	between the negative head and positive body of $r$, which can be, according
	to Lemma \ref{lemma:se:negative head repetition}, removed from the negative
	head of $r$ while preserving \SE-equivalence.  After these steps are
	performed, we obtain the rule
	\begin{equation} \label{eq:prop:se:rule equivalent to canonical:intermediate rule}
		(H(r)^+ \setminus B(r)^-); \lpnot (H(r)^- \setminus B(r)^+)
			\lpif B(r)^+, \lpnot B(r)^-.
	\end{equation}
	This is also the result of the defined transformation, unless the set
	$H(r)^+ \setminus B(r)^-$ is empty. In that case, one can repeatedly apply
	Lemma \ref{lemma:se:constraint} to move the atoms from the negative head of
	rule \eqref{eq:prop:se:rule equivalent to canonical:intermediate rule} into
	its positive body. In this case, the transformation returns the
	canonical rule
	\[
		\lpif (B(r)^+ \cup H(r)^-), \lpnot B(r)^-. \qed
	\]
\end{proof}

\begin{lemma*}{lemma:se:model conditions}
	Let $r$ be a rule. An \SE-interpretation $\an{I, J}$ is not an \SE-model of
	$r$ if and only if the following conditions are satisfied:
	\begin{enumerate}
		\item $H(r)^- \cup B(r)^+ \subseteq J$ and $J \subseteq \lang \setminus
			B(r)^-$.

		\item Either $J \subseteq \lang \setminus H(r)^+$ or both $B(r)^+
			\subseteq I$ and $I \subseteq \lang \setminus H(r)^+$.
	\end{enumerate}
\end{lemma*}
\begin{proof}
	\label{proof:lemma:se:model conditions}

	Suppose first that the above conditions hold. We will show that $\an{I, J}$
	is not an \SE-model of $r$. Due to the first condition, $r^J$ is equal to
	$H(r)^+ \lpif B(r)^+$ and due to the second condition, either $J$ is not a
	\C-model of $r$, or $I$ contains the body of $r^J$ but does not contain any
	atom from its head, which means $I$ is not a \C-model of $r^J$.
	Consequently, $\an{I, J}$ is not an \SE-model of $r$.

	Now suppose $I, J$ are two interpretations such that the above conditions do
	not hold. We will show that $\an{I, J}$ is an \SE-model of $r$. We need to
	consider the following four cases:
	\begin{enumerate}
		\renewcommand{\labelenumi}{\alph{enumi})}

		\item If $J$ does not contain some atom from the negative head of $r$ or
			it contains an atom from the negative body of $r$, then $J$ is a
			\C-model of $r$ and $r^J$ is $\varepsilon$, so $I$ is a \C-model of
			$r^J$.  Consequently, $\an{I, J}$ is an \SE-model of $r$.

		\item If $J$ does not contain some atom from the positive body of $r$,
			then $J$ is a \C-model of $r$ and $I$ is a \C-model of $r^J$ due to the
			fact that $I$ is a subset of $J$. Consequently, $\an{I, J}$ is an
			\SE-model of $r$.

		\item If $J$ contains an atom from the positive head of $r$ and $I$ does
			not include the positive body of $r$, then $J$ is a \C-model of $r$ and
			$I$ is a \C-model of $r^J$. Consequently, $\an{I, J}$ is an \SE-model of
			$r$.

		\item If $J$ contains some atom from the positive head of $r$ and $I$ also
			contains some atom from the positive head of $r$, then $J$ is a \C-model
			of $r$ and $I$ is a \C-model of $r^J$. Consequently, $\an{I, J}$ is an
			\SE-model of $r$. \qed
	\end{enumerate}
\end{proof}

\begin{corollary} \label{cor:se:model conditions}
	Let $r$ be a canonical rule different from $\varepsilon$, put $I = B(r)^+$,
	$J = H(r)^- \cup B(r)^+$ and $J' = \lang \setminus B(r)^-$, and let $\at$ be
	an atom. Then the following holds:
	\begin{enumerate}
		\renewcommand{\labelenumi}{(\arabic{enumi})}
		\item The \SE-interpretation $\an{I, J}$ is not an \SE-model of $r$.
			\label{part:cor:se:model conditions:1}
		\item The \SE-interpretation $\an{I, J \cup \set{\at}}$ is an \SE-model of
			$r$ if and only if $\at$ belongs to $B(r)^-$.
			\label{part:cor:se:model conditions:2}
		\item The \SE-interpretation $\an{I \cup \set{\at}, J \cup \set{\at}}$ is
			an \SE-model of $r$ if and only if $\at$ belongs to $H(r)^+ \cup
			B(r)^-$.
			\label{part:cor:se:model conditions:3}
		\item The \SE-interpretation $\an{I, J'}$ is not an \SE-model of $r$.
			\label{part:cor:se:model conditions:4}
	\end{enumerate}
\end{corollary}
\begin{proof}
	All parts of the Corollary easily follow from Lemma \ref{lemma:se:model
	conditions} and the disjointness properties satisfied by canonical
	rules. \qed
\end{proof}

\begin{lemma} \label{lemma:se:negative body}
	Let $r$ be a canonical rule different from $\varepsilon$ and $\mcS$ be the
	set of \SE-models of $r$. An atom $\at$ belongs to $B(r)^-$ if and only if
	for all $\an{I, J} \in \seint$,
	\begin{equation*}
		p \in J \text{ implies } \an{I, J} \in \mcS \enspace.
		\tag{\ref{eq:lemma:se:negative body}}
	\end{equation*}
\end{lemma}
\begin{proof}
	Suppose $\at$ belongs to $B(r)^-$ and take some \SE-interpretation $\an{I,
	J}$ such that $\at$ is in $J$. Then $J$ is a \C-model of $r$ and $r^J$ is
	equal to $\varepsilon$, so $I$ is a \C-model of $r^J$. Hence, $\an{I, J}$ is
	an \SE-model of $r$, and since the choice of $\an{I, J}$ was arbitrary, we
	conclude that $\at$ satisfies condition \eqref{eq:lemma:se:negative body}.

	Now let $I_0 = B(r)^+$ and $J_0 = H(r)^- \cup B(r)^+$ and suppose $\at$ is
	an atom satisfying condition \eqref{eq:lemma:se:negative body}. Then the
	\SE-interpretation $\an{I_0, J_0 \cup \set{\at}}$ must belong to $\mcS$ and
	by Corollary \ref{cor:se:model conditions}\eqref{part:cor:se:model
	conditions:2} we conclude that $\at$ belongs to $B(r)^-$. \qed
\end{proof}

\begin{lemma} \label{lemma:se:positive head}
	Let $r$ be a canonical rule different from $\varepsilon$ and $\mcS$ be the
	set of \SE-models of $r$. An atom $\at$ belongs to $H(r)^+$ if and only if
	it does not belong to $B(r)^-$ and for all $\an{I, J} \in \seint$,
	\begin{equation*}
		p \in I \text{ implies } \an{I, J} \in \mcS \enspace.
		\tag{\ref{eq:lemma:se:positive head}}
	\end{equation*}
\end{lemma}
\begin{proof}
	Suppose $\at$ belongs to $H(r)^+$. Since $r$ is a canonical rule, $\at$
	does not belong to $B(r)^-$. Take some \SE-interpretation $\an{I, J}$ such
	that $\at$ belongs to $I$. Then $\at$ must also belong to $J$, so $J$ is a
	\C-model of $r$ and, for the same reason, $I$ is a \C-model of $r^J$.
	Consequently, $\an{I, J}$ is an \SE-model of $r$. Since the choice of
	$\an{I, J}$ was arbitrary, we conclude that $\at$ satisfies condition
	\eqref{eq:lemma:se:positive head}.

	Now let $I_0 = B(r)^+$ and $J_0 = H(r)^- \cup B(r)^+$ and suppose $\at$ is
	an atom satisfying condition \eqref{eq:lemma:se:positive head}. Then the
	\SE-interpretation $\an{I_0 \cup \set{\at}, J_0 \cup \set{\at}}$ must belong
	to $\mcS$ and by Corollary \ref{cor:se:model
	conditions}\eqref{part:cor:se:model conditions:3} we conclude that $\at$
	belongs to $H(r)^+ \cup B(r)^-$. Moreover, by assumption we know that $\at$
	does not belong to $B(r)^-$, so it belongs to $H(r)^+$. \qed
\end{proof}

\begin{lemma*}{lemma:se:negative body and positive head}
	An atom $\at$ belongs to $B(r)^-$ if and only if for all $\an{I, J} \in
	\seint$, the condition \eqref{eq:lemma:se:negative body} is satisfied.
	
	An atom $\at$ belongs to $H(r)^+$ if and only if it does not belong to
	$B(r)^-$ and for all $\an{I, J} \in \seint$, the condition
	\eqref{eq:lemma:se:positive head} is satisfied.
\end{lemma*}
\begin{proof}
	\label{proof:lemma:se:negative body and positive head}
	Follows from Lemmas \ref{lemma:se:negative body} and \ref{lemma:se:positive
	head}.
\end{proof}

\begin{lemma} \label{lemma:se:positive body}
	Let $r$ be a canonical rule different from $\varepsilon$ and $\mcS$ be
	the set of \SE-models of $r$. An atom $\at$ belongs to $B(r)^+$ if and only
	if for all $\an{I, J} \in \seint$ the following conditions are satisfied:
	\begin{align*}
		p \notin J &\text{ implies } \an{I, J} \in \mcS \enspace;
			\tag{\ref{eq:lemma:se:negative head}} \\
		J \cap H(r)^+ \neq \emptyset \text{ and } p \notin I
			&\text{ implies } \an{I, J} \in \mcS \enspace.
			\tag{\ref{eq:lemma:se:positive body}}
	\end{align*}
\end{lemma}
\begin{proof}
	Suppose $\at$ belongs to $B(r)^+$ and take some \SE-interpretation $\an{I,
	J}$ such that $\at$ is not in $J$. Since $I$ is a subset of $J$, we obtain
	$\at$ is not in $I$ either. Hence, $J$ is a \C-model of $r$ and $I$ is a
	\C-model of $r^J$ and we conclude that $\an{I, J}$ is an \SE-model of $r$.
	The choice of $\an{I, J}$ was arbitrary, so this implies that condition
	\eqref{eq:lemma:se:negative head} is satisfied for $\at$.

	Now take some \SE-interpretation $\an{I, J}$ such that $J \cap H(r)^+ \neq
	\emptyset$ and $p$ is not in $I$. From the former it follows that $J$ is a
	\C-model of $r$ and from the latter that $I$ is a \C-model of $r^J$. Thus,
	$\an{I, J}$ is an \SE-model of $r$ and since the choice of $\an{I, J}$ was
	arbitrary, we conclude that condition \eqref{eq:lemma:se:positive body}
	holds for $\at$.

	For the converse implication, suppose $\at$ is an atom satisfying conditions
	\eqref{eq:lemma:se:negative head} and \eqref{eq:lemma:se:positive body}. We
	consider two cases:
	\begin{enumerate}
		\renewcommand{\labelenumii}{\alph{enumii})}

		\item If $H(r)^+$ is empty, then since $r$ is canonical, we know that
			$H(r)^-$ is also empty. So according to Corollary \ref{cor:se:model
			conditions}\eqref{part:cor:se:model conditions:1}, the
			\SE-interpretation $\an{I_0, J_0}$, where $I_0 = J_0 = B(r)^+$, does not
			belong to $\mcS$. Furthermore, by condition \eqref{eq:lemma:se:negative
			head} we can conclude that $\an{I_0 \setminus \set{\at}, J_0 \setminus
			\set{\at}}$ belongs to $\mcS$. Thus, $J_0$ must be different from $J_0
			\setminus \set{\at}$, so $\at$ must belong to $J_0 = B(r)^+$.
			
		\item If $H(r)^+$ is nonempty, then it follows from Corollary
			\ref{cor:se:model conditions}\eqref{part:cor:se:model conditions:4} that
			the \SE-interpretation $\an{I_0, J_0}$, where $I_0 = B(r)^+$ and $J_0 =
			\lang \setminus B(r)^-$, does not belong to $\mcS$. We can also conclude
			that $J$ contains some atom from $H(r)^+$ because, since $r$ is
			canonical, $H(r)^+$ is disjoint from $B(r)^-$. Thus, by condition
			\eqref{eq:lemma:se:positive body} we conclude that $\an{I_0 \setminus
			\set{\at}, J_0}$ belongs to $\mcS$. Consequently, $I_0$ must be
			different from $I_0 \setminus \set{\at}$, so $\at$ belongs to $I_0 =
			B(r)^+$. \qed
	\end{enumerate}
\end{proof}

\begin{lemma} \label{lemma:se:negative head}
	Let $r$ be a canonical rule different from $\varepsilon$ and $\mcS$ be
	the set of \SE-models of $r$.An atom $\at$ belongs to $H(r)^-$ if and only
	if it does not belong to $B(r)^+$ and for all $\an{I, J} \in \seint$,
	\begin{equation*}
		p \notin J \text{ implies } \an{I, J} \in \mcS \enspace.
		\tag{\ref{eq:lemma:se:negative head}}
	\end{equation*}
\end{lemma}
\begin{proof}
	Suppose $\at$ is some atom from $H(r)^-$ and take some \SE-interpretation
	$\an{I, J}$ such that $\at$ is not in $J$. Then $J$ is a \C-model of $r$ and
	$r^J$ is equal to $\varepsilon$ so $I$ is a \C-model of $r^J$.
	Consequently, $\an{I, J}$ is an \SE-model of $r$ and since the choice of
	$\an{I, J}$ was arbitrary, we conclude that condition
	\eqref{eq:lemma:se:negative head} is satisfied for $\at$.

	Now let $I_0 = B(r)^+$ and $J_0 = H(r)^- \cup B(r)^+$ and suppose $\at$ is
	an atom that does not belong to $B(r)^+$ and it satisfies condition
	\eqref{eq:lemma:se:negative head}. Corollary \ref{cor:se:model
	conditions}\eqref{part:cor:se:model conditions:1} guarantees that the
	\SE-interpretation $\an{I_0, J_0}$ is not an \SE-model of $r$. Furthermore,
	from condition \eqref{eq:lemma:se:negative head} we obtain that the
	\SE-interpretation $\an{I_0 \setminus \set{\at}, J_0 \setminus \set{\at}}$
	belongs to $\mcS$. Thus, $J_0$ must differ from $J_0 \setminus \set{\at}$,
	which implies that $\at$ belongs to $J_0$. Furthermore, since $J_0 = H(r)^-
	\cup B(r)^+$ and $\at$ does not belong to $B(r)^+$, we conclude that $\at$
	belongs to $H(r)^-$. \qed
\end{proof}

\begin{lemma*}{lemma:se:positive body and negative head}
	An atom $\at$ belongs to $B(r)^+$ if and only if for all $\an{I, J} \in
	\seint$, the conditions \eqref{eq:lemma:se:negative head} and
	\eqref{eq:lemma:se:positive body} are satisfied.

	An atom $\at$ belongs to $H(r)^-$ if and only if it does not belong to
	$B(r)^+$ and for all $\an{I, J} \in \seint$, the condition
	\eqref{eq:lemma:se:negative head} is satisfied.
\end{lemma*}
\begin{proof}
	\label{proof:lemma:se:positive body and negative head}
	Follows from Lemmas \ref{lemma:se:positive body} and \ref{lemma:se:negative
	head}.
\end{proof}

\begin{theorem*}{thm:se:canonical from models}
	For every canonical rule $r$, $\synt{\modse{r}} = r$.
\end{theorem*}
\begin{proof} \label{proof:thm:se:canonical from models}
	If $r$ is equal to $\varepsilon$, then $\modse{r} = \seint$ and by
	Definition \ref{def:se:models to rule}, the rule $\synt{\seint}$ is equal to
	$\varepsilon$ so the identity is satisfied.

	In the principal case, $r$ is a canonical rule different from $\varepsilon$.
	Let $\mcS$ be the set of \SE-models of $r$. It follows from Definition
	\ref{def:se:models to rule} and Lemmas \ref{lemma:se:negative body} to
	\ref{lemma:se:negative head} that $r = \synt{\mcS}$. \qed
\end{proof}

\begin{lemma} \label{lemma:se:induced sets disjoint}
	Let $\mcS$ be a set of \SE-interpretations different from $\seint$. Then the
	sets of $H(\mcS)^+ \cup H(\mcS)^-$, $B(\mcS)^+$ and $B(\mcS)^-$ are pairwise
	disjoint.
\end{lemma}
\begin{proof}
	Suppose that $\at$ is a member of both $H(\mcS)^+ \cup H(\mcS)^-$ and
	$B(\mcS)^+$. Then, since $\at$ is an \SPosBody{} atom, it cannot be an
	\SNegHead{} atom by definition. Thus, $\at$ belongs to both $H(\mcS)^+$ and
	$B(\mcS)^+$. We will show that this is impossible given our assumption that
	$\mcS$ is different from $\seint$. Take an arbitrary \SE-interpretation
	$\an{I, J}$. If $\at$ belongs to $I$, then since $\at$ is an \SPosHead{}
	atom, $\an{I, J}$ belongs to $\mcS$. If $\at$ does not belong to $I$ but it
	belongs to $J$, then $J$ contains the \SPosHead{} atom $\at$, so since $\at$
	is an \SPosBody{} atom, $\an{I, J}$ belongs to $\mcS$. Finally, if $p$ does
	not belong to $J$, then since $\at$ is an \SPosBody{} atom, $\an{I, J}$
	belongs to $\mcS$. This means that $\mcS$ must contain all
	\SE-interpretations and is in conflict with our assumption.

	Now suppose that $\at$ is a member of both $H(\mcS)^+ \cup H(\mcS)^-$ and
	$B(\mcS)^-$. Then, since $\at$ is an \SNegBody{} atom, it cannot be an
	\SPosHead{} atom by definition. Thus, $\at$ belongs to both $H(\mcS)^-$ and
	$B(\mcS)^-$. We will show that this is impossible given our assumption that
	$\mcS$ is different from $\seint$. Take an arbitrary \SE-interpretation
	$\an{I, J}$. If $\at$ belongs to $J$, then since $\at$ is an \SNegBody{}
	atom, $\an{I, J}$ belongs to $\mcS$. On the other hand, if $\at$ does not
	belong to $J$, then since $\at$ is an \SNegHead{} atom, $\an{I, J}$ belongs
	to $\mcS$. This means that $\mcS$ must contain all \SE-interpretations and
	is in conflict with our assumption to the contrary.

	Next, suppose that $\at$ is a member of both $B(\mcS)^+$ and $B(\mcS)^-$.
	By the same arguments as in the previous case, this implies that $\mcS$ must
	be equal to $\seint$, contrary to the assumption.
\end{proof}

\begin{lemma} \label{lemma:se:canonical induced by nonrepresentable}
	For every set of \SE-interpretations $\mcS$, $\synt{\mcS}$ is a canonical
	rule.
\end{lemma}
\begin{proof}
	\label{proof:lemma:se:canonical induced by nonrepresentable}

	If $\mcS$ is equal to $\seint$, then $\synt{\mcS}$ is equal to the
	canonical rule $\varepsilon$ and the proof is finished. Otherwise,
	$\synt{\mcS}$ is of the form
	\[
	H(\mcS)^+; \lpnot H(\mcS)^- \lpif B(\mcS)^+, \lpnot B(\mcS)^-.
	\]
	To show that this rule is canonical, we need to prove that the following
	conditions are satisfied:
	\begin{enumerate}
		\item The sets $H(\mcS)^+ \cup H(\mcS)^-$, $B(\mcS)^+$ and $B(\mcS)^-$ are
			pairwise disjoint.
		\item If $H(\mcS)^+$ is empty, then $H(\mcS)^-$ is also empty.
	\end{enumerate}
	The first condition follows from Lemma \ref{lemma:se:induced sets disjoint}.
	To prove the second condition, suppose $H(\mcS)^+$ is empty. Then by
	definition $B(\mcS)^+$ contains all atoms whose absence in $J$ implies that
	$\an{I, J}$ belongs to $\mcS$. By definition, then, $H(\mcS)^-$ stays empty.
	\qed
\end{proof}

\begin{lemma*}{lemma:se:models induced least}
	The set of all \SE-models of a canonical rule $r$ is the least among
	all sets of \SE-interpretations $\mcS$ such that $\synt{\mcS} = r$.
\end{lemma*}
\begin{proof}
	\label{proof:lemma:se:models induced least}

	Let $r$ be a canonical rule with the set of \SE-models $\mcS_r$. From
	Theorem \ref{thm:se:canonical from models} we know that $\synt{\mcS_r} = r$,
	so it remains to show that $\mcS_r$ is a subset of every set of
	\SE-interpretations $\mcS$ such that $\synt{\mcS} = r$. Take one such
	$\mcS$. In case $r$ is the canonical tautology $\at_\varepsilon \lpif
	\at_\varepsilon$, it follows that $H(\mcS)^+ = B(\mcS)^+ =
	\set{\at_\varepsilon}$. According to Lemma \ref{lemma:se:induced sets
	disjoint}, this is possible only in case $\mcS = \seint = \mcS_r$, so it
	trivially holds that $\mcS_r$ is a subset of $\mcS$.

	In the principal case, $r$ is different from the canonical tautology, so
	$\mcS$ must be different from $\seint$ and from $\synt{\mcS} = r$ we obtain
	that $H(r)^+ = H(\mcS)^+$, $H(r)^- = H(\mcS)^-$, $B(r)^+ = B(\mcS)^+$ and
	$B(r)^- = B(\mcS)^-$. Let $\an{I, J}$ be an \SE-model of $r$. Then one of
	the conditions of Lemma \ref{lemma:se:model conditions} must be violated. We
	distinguish the following four possible violations:
	\begin{enumerate}
		\renewcommand{\labelenumi}{\alph{enumi})}

		\item If $H(\mcS)^- \cup B(\mcS)^+ \nsubseteq J$, then $J$ does not
			contain some atom from $H(\mcS)^- \cup B(\mcS)^+$.  From the definitions
			of \SNegHead{} atoms and \SPosBody{} atoms we then obtain that $\an{I,
			J}$ belongs to $\mcS$.

		\item If $J \nsubseteq \lang \setminus B(\mcS)^-$, then $J$ contains some
			atom from $B(\mcS)^-$. From the definition of \SNegBody{} atoms we then
			infer that $\an{I, J}$ belongs to $\mcS$.

		\item If $J \nsubseteq \lang \setminus H(\mcS)^+$ and $B(\mcS)^+
			\nsubseteq I$, then $J$ contains some atom from $H(\mcS)^+$ and $I$ does
			not contain some atom from $B(\mcS)^+$. By the definition of \SPosBody{}
			atoms, $\an{I, J}$ belongs to $\mcS$.

		\item If $J \nsubseteq \lang \setminus H(\mcS)^+$ and $I \nsubseteq \lang
			\setminus H(\mcS)^+$, then $I$ contains some atom from $H(\mcS)^+$. By
			the definition of \SPosHead{} atoms, $\an{I, J}$ belongs to $\mcS$.
			\qed
	\end{enumerate}
\end{proof}

\begin{proposition} \label{prop:se:rule representable:1}
	A set of \SE-interpretations $\mcS$ is rule-representable if and only if all
	\SE-interpretations from $\mcS$ are \SE-models of $\synt{\mcS}$.
\end{proposition}
\begin{proof}
	If $\mcS$ is a rule-representable set of \SE-interpretations, then there
	exists some rule $r$ such that $\mcS = \modse{r}$. Let $r'$ be the
	canonical rule $\secan{r}$. According to Theorem \ref{thm:se:canonical
	equivalence}, $\mcS = \modse{r'}$, and so Theorem \ref{thm:se:canonical from
	models} implies that $\synt{\mcS} = \synt{\modse{r'}} = r'$. Thus, all
	\SE-interpretations from $\mcS$ are \SE-models of $\synt{\mcS}$.

	On the other hand, if all \SE-interpretations in $\mcS$ are \SE-models of
	the rule $r = \synt{\mcS}$, then $\mcS$ is a subset of $\modse{r}$. Also, by
	Lemma \ref{lemma:se:canonical induced by nonrepresentable} it follows that
	$r$ is canonical and so Lemma \ref{lemma:se:models induced least}
	implies that $\modse{r}$ is a subset of $\mcS$. Consequently, $\mcS =
	\modse{r}$. \qed
\end{proof}

\begin{proposition} \label{prop:se:rule representable:2}
	A set of \SE-interpretations $\mcS$ is rule-representable if and only if
	there exist convex sublattices $L_1, L_2$ of $\an{\pint, \subseteq}$ such
	that the complement of $\mcS$ relative to $\seint$ is equal to
	\[
		\Set{\an{I, J} \in \seint | I \in L_1 \land J \in L_2} \cup
			\Set{\an{I, J} \in \seint | J \in L_1 \cap L_2} \enspace.
	\]
\end{proposition}
\begin{proof}
	Suppose that $\mcS$ is a rule-representable set of \SE-interpretations. Then
	there exists some rule $r$ such that $\mcS = \modse{r}$. Let the sets of
	interpretations $L_1$, $L_2$ be defined as follows:
	\begin{align*}
		L_1 &= \Set{ I \in \pint | B(r)^+ \subseteq I \subseteq \lang \setminus
		H(r)^+ } \\
		L_2 &= \Set{ J \in \pint | H(r)^- \cup B(r)^+ \subseteq J \subseteq \lang
		\setminus B(r)^- }
	\end{align*}
	It can be straightforwardly verified that these sets are convex sublattices
	of $\an{\pint, \subseteq}$. It remains to prove that the complement of
	$\mcS$ relative to $\seint$ is equal to the set of \SE-interpretations
	\begin{equation} \label{eq:proof:thm:se:rule representable:1}
		\Set{\an{I, J} \in \seint | I \in L_1 \land J \in L_2} \cup
			\Set{\an{I, J} \in \seint | J \in L_1 \cap L_2} \enspace.
	\end{equation}
	According to Lemma \ref{lemma:se:model conditions}, an \SE-interpretation
	$\an{I, J}$ does not belong to $\mcS$ if and only if these two conditions
	are satisfied:
	\begin{enumerate}
		\item $H(r)^- \cup B(r)^+ \subseteq J$ and $J \subseteq \lang \setminus
			B(r)^-$.
		\item Either $J \subseteq \lang \setminus H(r)^+$ or both $B(r)^+
			\subseteq I$ and $I \subseteq \lang \setminus H(r)^+$.
	\end{enumerate}
	It is not difficult to see that whenever the first condition and first
	disjunct of the second condition are satisfied, $\an{I, J}$ belongs to the
	second part of the set \eqref{eq:proof:thm:se:rule representable:1}.
	Similarly, the first condition together with the second disjunct of the
	second condition imply that $\an{I, J}$ belongs to the first part of the set
	\eqref{eq:proof:thm:se:rule representable:1}. Conversely, given the
	definitions of $L_1$ and $L_2$, it is easy to see that any
	\SE-interpretation belonging to the set \eqref{eq:proof:thm:se:rule
	representable:1} satisfies the conditions of Lemma \ref{lemma:se:model
	conditions}. Thus, the set \eqref{eq:proof:thm:se:rule representable:1}
	coincides with the complement of $\mcS$ relative to $\seint$.

	Now suppose that $L_1, L_2$ are two convex sublattices of $\an{\pint,
	\subseteq}$ such that the complement of $\mcS$ relative to $\seint$ is equal
	to the set \eqref{eq:proof:thm:se:rule representable:1}. Let $\top_1$,
	$\bot_1$ be the top and bottom elements of $L_1$ and $\top_2$, $\bot_2$ be
	the top and bottom elements of $L_2$. Furthermore, let $r$ be a rule of the
	form
	\[
	H^+; \lpnot H^- \lpif B^+, \lpnot B^-.
	\]
	where $H^+ = \lang \setminus \top_1$, $H^- = \bot_2$, $B^+ = \bot_1$ and
	$B^- = \lang \setminus \top_2$. We will show that $\mcS = \modse{r}$.

	Suppose first that the \SE-interpretation $\an{I, J}$ is not an \SE-model of
	$r$. Then, by Lemma \ref{lemma:se:model conditions}, $J$ includes $H^- \cup
	B^+ = \bot_2 \cup \bot_1$ and $J$ is included in $\lang \setminus B^- =
	\lang \setminus (\lang \setminus \top_2) = \top_2$. By convexity of $L_2$ we
	now obtain that $J$ belongs to $L_2$. Lemma \ref{lemma:se:model conditions}
	also implies that either $J$ is included in $\lang \setminus H^+ = \top_1$,
	or $I$ includes $B^+ = \bot_1$ and is included in $\lang \setminus H^+ =
	\top_1$. The convexity of $L_1$ now implies that in the former case $J$
	belongs to $L_1$, while in the latter case $I$ belongs $L_1$. In any of
	these cases, $\an{I, J}$ is a member of the set \eqref{eq:proof:thm:se:rule
	representable:1}.

	Now let $\an{I, J}$ be some \SE-interpretation not belonging to $\mcS$. If
	$\an{I, J}$ belongs to the first part of the set \eqref{eq:proof:thm:se:rule
	representable:1}, then $\bot_1 \subseteq I \subseteq \top_1$ and $\bot_2
	\subseteq J \subseteq \top_2$. Thus, $I$ includes $B^+$ and is included in
	$\lang \setminus H^+$, and $J$ includes $H^-$ and is included in $\lang
	\setminus B^-$. Also, since $I$ is a subset of $J$, $J$ includes $B^+$.
	Lemma \ref{lemma:se:model conditions} then implies that $\an{I, J}$ is not
	an \SE-model of $r$. If $\an{I, J}$ belongs to the second part of the set
	\eqref{eq:proof:thm:se:rule representable:1}, then $\bot_1 \subseteq J
	\subseteq \top_1$ and $\bot_1 \subseteq J \subseteq \top_2$. Thus, $J$
	includes both $H^-$ and $B^+$ and is included in $\lang \setminus B^-$ and
	in $\lang \setminus H^+$. As a consequence of Lemma \ref{lemma:se:model
	conditions}, $\an{I, J}$ is not an \SE-model of $r$. \qed
\end{proof}

\begin{theorem*}{thm:se:rule representable}
	Let $\mcS$ be a set of \SE-interpretations. Then the following conditions
	are equivalent:
	\begin{enumerate}
		\item The set of \SE-interpretations $\mcS$ is rule-representable.
		\item All \SE-interpretations from $\mcS$ are \SE-models of $\synt{\mcS}$.
		\item There exist convex sublattices $L_1, L_2$ of $\an{\pint, \subseteq}$
			such that the complement of $\mcS$ relative to $\seint$ is equal to
			\[
				\Set{\an{I, J} \in \seint | I \in L_1 \land J \in L_2} \cup
				\Set{\an{I, J} \in \seint | J \in L_1 \cap L_2} \enspace.
			\]
	\end{enumerate}
\end{theorem*}
\begin{proof} \label{proof:thm:se:rule representable}
	Follows from Propositions \ref{prop:se:rule representable:1} and
	\ref{prop:se:rule representable:2}. \qed
\end{proof}

\begin{theorem*}{thm:se:equivalence comparison}
	\SU-equivalence is stronger than \SR-equivalence, which itself is stronger
	than \SMR-equivalence, which in turn is stronger than strong equivalence.
	That is,
	\[
		\equiv_{\mSU} \,\succ\, \equiv_{\mSR}
			\,\succ\, \equiv_{\mSMR} \,\succ\, \equiv_{\mSt} \enspace.
	\]
\end{theorem*}
\begin{proof}
	\label{proof:thm:se:equivalence comparison}

	We first need to show that if two programs are \SU-equivalent, they are also
	\SR-equivalent, but the converse does not hold. Suppose $\prP_1$, $\prP_2$
	are \SU-equivalent programs. Then, according to Theorem 4.3 in
	\cite{Inoue2004}, their symmetric difference $(\prP_1 \setminus \prP_2) \cup
	(\prP_2 \setminus \prP_1)$ contains only \SE-tautological
	rules.\footnote{The Theorem actually states that the symmetric difference
	contains only \emph{valid} rules. A rule is valid, as defined in
	\cite{Inoue2004}, if and only if it is \SE-tautological.} To show that
	$\prP_1$ is \SR-equivalent to $\prP_2$, suppose $\mcS$ is a set of
	\SE-interpretations belonging to the set
	\begin{equation} \label{eq:proof:thm:se:equivalence comparison:1}
		\Set{ \modse{r} | r \in \prP_1 \cup \set{\varepsilon} } \enspace.
	\end{equation}
	Then there exists some rule $r$ with $\mcS = \modse{r}$ that either belongs
	to $\prP_1$, or is \SE-tautological. Furthermore, $\prP_1 = (\prP_1 \cap
	\prP_2) \cup (\prP_1 \setminus \prP_2)$, so $r$ either belongs to $\prP_2$,
	or it belongs to $\prP_1 \setminus \prP_2$, or it is \SE-tautological. But
	all rules from $\prP_1 \setminus \prP_2$ are \SE-tautological, so we can
	conclude that $r$ either belongs to $\prP_2$ or it is \SE-tautological.
	Consequently, $\mcS$ belongs to the set
	\begin{equation} \label{eq:proof:thm:se:equivalence comparison:2}
		\Set{ \modse{r} | r \in \prP_2 \cup \set{\varepsilon} } \enspace.
	\end{equation}
	A similar argument yields that the set \eqref{eq:proof:thm:se:equivalence
	comparison:2} is a subset of the set \eqref{eq:proof:thm:se:equivalence
	comparison:1}. Consequently, the two sets are equal, so $\prP_1$ is
	\SR-equivalent to $\prP_2$.

	To see that the converse does not hold, take the programs $\prP_1 =
	\set{\lpnot p \lpif.}$ and $\prP_2 = \set{\lpif p.}$. It can be easily
	verified that they are \SR-equivalent, but since their symmetric difference
	contains rules that are not \SE-tautological, they are not \SU-equivalent
	(according to Theorem 4.3 in \cite{Inoue2004}).

	Next, need to show that if two programs are \SR-equivalent, they are also
	\SMR-equivalent, but the converse does not hold. It can be immediately seen
	that
	\[
		\Set{ \modse{r} | r \in \prP_1 \cup \set{\varepsilon} }
			= \Set{ \modse{r} | r \in \prP_2 \cup \set{\varepsilon} }
	\]
	implies
	\[
		\min \Set{ \modse{r} | r \in \prP_1 \cup \set{\varepsilon} }
			= \min \Set{ \modse{r} | r \in \prP_2 \cup \set{\varepsilon} } \enspace,
	\]
	so the first part of the proof is finished. As for the second part, it
	suffices to consider programs $\prP_1 = \Set{p.}$ and $\prP_2 = \Set{p., p \lpif
	q.}$ which are \SMR-equivalent, but not \SR-equivalent.

	Finally, we need to prove that if two programs are \SMR-equivalent, they are
	also strongly equivalent, but not vice versa. So take some \SMR-equivalent
	programs $\prP_1$, $\prP_2$. Then
	\begin{equation} \label{eq:proof:thm:se:equivalence comparison:3}
		\min \Set{ \modse{r} | r \in \prP_1 \cup \set{\varepsilon} }
			= \min \Set{ \modse{r} | r \in \prP_2 \cup \set{\varepsilon} } \enspace.
		\end{equation}
	Furthermore,
	\begin{align*}
		\modse{\prP_1}
		& = \bigcap \Set{ \modse{r} | r \in \prP_1 } \\
		& = \bigcap \left( \Set{ \modse{r} | r \in \prP_1 } \cup \Set{\seint}
		  	\right) \\
		& = \bigcap \Set{ \modse{r} | r \in \prP_1 \cup \set{\varepsilon} }
				\enspace,
	\end{align*}
	and whenever some set of \SE-interpretations $\mcS$ is non-minimal within
	\begin{equation} \label{eq:proof:thm:se:equivalence comparison:4}
		\Set{ \modse{r} | r \in \prP_1 \cup \set{\varepsilon} } \enspace,
	\end{equation}
	there exists some set of \SE-interpretations $\mcT$ from 
	\eqref{eq:proof:thm:se:equivalence comparison:4} such that $\mcT \subsetneq
	\mcS$. Thus, $\mcT \cap \mcS = \mcT$, and so such non-minimal sets are
	irrelevant when determining the intersection of all sets in the set
	\eqref{eq:proof:thm:se:equivalence comparison:4}. Consequently,
	\[
		\modse{\prP_1}
			= \bigcap \min \Set{ \modse{r} | r \in \prP_1 \cup \set{\varepsilon} }
			\enspace.
	\]
	By similar arguments we obtain that
	\[
		\modse{\prP_2}
			= \bigcap \min \Set{ \modse{r} | r \in \prP_2 \cup \set{\varepsilon} }
			\enspace.
	\]
	Thus, \eqref{eq:proof:thm:se:equivalence comparison:3} implies that $\prP_1$
	is strongly equivalent to $\prP_2$.

	To see that the converse does not hold, consider programs $\prP_1 = \set{p.,
	q.}$ and $\prP_2 = \set{p \lpif q., q.}$, which are strongly equivalent, but
	not \SMR-equivalent. \qed
\end{proof}

\begin{proposition}
	If $\mcS_1, \mcS_2$ are rule-representable sets of \SE-models, then $\mcS_1
	\cup \mcS_2$ is also rule-representable.
\end{proposition}
\begin{proof}
	Let $L_1^I, L_1^J, L_2^I, L_2^J$ be convex sublattices of $\an{\pint,
	\subseteq}$ such that
	\begin{align*}
		\mcS_1 &= \Set{ \an{I, J} | I \in L_1^I \land J \in L_1^J }
			\cup \Set{ \an{I, J} | J \in L_1^I \cap L_1^J } \enspace, \\
		\mcS_2 &= \Set{ \an{I, J} | I \in L_2^I \land J \in L_2^J }
			\cup \Set{ \an{I, J} | J \in L_2^I \cap L_2^J } \enspace.
	\end{align*}
	Furthermore, let $\mcS = \mcS_1 \cap \mcS_2$ and
	\begin{align*}
		L^I &= \Set{ I | (\exists J \in \pint)( \an{J, J} \notin \mcS_1 \cap
		\mcS_2 \land \an{I, J} \in \mcS_1 \cap \mcS_2 } \\
		L^J &= \Set{ J | (\exists I \in \pint)( \an{I, J} \in \mcS_1 \cap \mcS_2 }
	\end{align*}
	We need to prove that $L^I, L^J$ are convex sublattices of $\an{\pint,
	\subseteq}$ and that
	\[
		\mcS = \Set{ \an{I, J} | I \in L^I \land J \in L^J }
			\cup \Set{ \an{I, J} | J \in L^I \cap L^J } \enspace.
	\]
	
\end{proof}

\end{extended}

\end{document}